\documentclass[letterpaper,twocolumn, journal]{IEEEtran}

\let\labelindent\relax
\usepackage{fullpage,graphicx,url,setspace}

\usepackage{xparse}
\usepackage{stmaryrd}
\usepackage{cite}
\usepackage{bm}
\usepackage{color}
\usepackage[small,bf]{caption}
\usepackage{amsmath,verbatim,amssymb,amsfonts,amscd, graphicx, algorithmic, algorithm}

\usepackage{amsmath,amsthm}
\usepackage{mathtools}
\mathtoolsset{mathic}
\usepackage{microtype}
\usepackage{wrapfig}

\usepackage{ifxetex,ifluatex}
\ifxetex
\usepackage{xltxtra}
\else
\ifluatex
\usepackage{luatextra}
\else
\usepackage{amssymb}
\usepackage{algorithmic, algorithm}
\usepackage[utf8]{inputenc}

\csname else\expandafter\endcsname
\romannumeral -`0
\fi
\fi
\iftrue\relax%
\defaultfontfeatures{Ligatures=TeX}
\usepackage[math-style=TeX, vargreek-shape=TeX]{unicode-math}
\fi

\usepackage{tikz}
\usetikzlibrary{shapes.geometric, positioning, shadows}
\usepackage{graphicx,color}

\newtheorem{theorem}{Theorem}

\newtheorem{lemma}{Lemma}

\newtheorem{remark}{Remark}

\newtheorem{corollary}{Corollary}
\usepackage{enumitem}
\newlist{enumerate*}{enumerate*}{1}
\setlist[enumerate*]{label=(\arabic*)}

\newcommand{\transpose}{^{\intercal}}

\newcommand{\vect}{\boldsymbol} 

\newcommand{\tr}{\mbox{tr}}

\newcommand{\ben}{\begin{eqnarray}}

\newcommand{\een}{\end{eqnarray}}

\newcommand{\pen}{\textsf{pen}}
\newcommand{\myvec}{\mbox{vec}}
\newcommand{\diag}{\mbox{diag}}
\newcommand{\cA}{\mathcal{A}}

\title{Poisson Matrix Recovery and Completion}

\author{Yang Cao,\thanks{Yang Cao
    (Email: caoyang@gatech.edu) and Yao Xie (Email: yao.xie@isye.gatech.edu)
   are with the H. Milton Stewart School of
    Industrial and Systems Engineering, Georgia Institute of
    Technology, Atlanta, GA.
   This work is partially supported by NSF grant CCF-1442635 and CMMI-1538746. Parts of the paper have appeared in GlobalSIP 2013, GlobalSIP 2014, and ISIT 2015.}
    \quad \and Yao Xie 
} \date{\today}

\newcommand{\yx}[1]{{\color{black}{#1}}}
\newcommand{\yc}[1]{{\color{black}{#1}}}

\begin{document}
\maketitle

\begin{abstract}

We extend the theory of low-rank matrix recovery and completion to the case when Poisson observations for a linear combination or a subset of the entries of a matrix are available, which arises in various applications with count data. We consider the usual matrix recovery formulation through maximum likelihood with proper constraints on the matrix $M$ of size $d_1$-by-$d_2$, and establish theoretical upper and lower bounds on the recovery error. Our bounds for matrix completion are nearly optimal up to a factor on the order of $\mathcal{O}(\log(d_1 d_2))$. These bounds are obtained by combing techniques for compressed sensing for sparse vectors with Poisson noise and for analyzing low-rank matrices, as well as adapting the arguments used for one-bit matrix completion \cite{davenport20121} (although these two problems are different in nature) and the adaptation requires new techniques exploiting properties of the Poisson likelihood function and tackling the difficulties posed by the locally sub-Gaussian characteristic of the Poisson distribution. Our results highlight a few important distinctions of the Poisson case compared to the prior work including having to impose a minimum signal-to-noise requirement on each observed entry and a gap in the upper and lower bounds. We also develop a set of efficient iterative algorithms and demonstrate their good performance on synthetic examples and real data.

\end{abstract}


\section{Introduction}

Recovering a low-rank matrix $M$ with Poisson observations is a key problem that arises from various real-world applications with count data, such as nuclear medicine, low-dose x-ray imaging \cite{brady2009optical}, network traffic analysis \cite{poissonGBG2013}, and call center data \cite{ShenHuang2005}.  There the observations are Poisson counts whose intensities are determined by the matrix, either through a subset of its entries or linear combinations of its entries.

Thus far much success has been achieved in solving the matrix completion and recovery problems using nuclear norm minimization, partly inspired by the theory of compressed sensing \cite{CandesTao2006,donoho2006compressed}. It has been shown that when $M$ is low rank, it can be recovered from observations of a subset or a linear combination of its entries (see, e.g.\cite{candes2009exact, keshavan2010matrix, candes2010power, dai2010set, recht2010guaranteed, recht2011simpler, cai2010singular, lin2009fast, mazumder2010spectral}).
Earlier work on matrix completion typically assume that the observations are noiseless, i.e., we may directly observe a subset of entries of $M$. In the real world, however, the observations are noisy, which is the focus of  the subsequent work \cite{keshavan2009matrix, candes2010matrix, negahban2011estimation, negahban2012restricted,rohde2011estimation,sonierror}, most of which consider a scenario when the observations are contaminated by Gaussian noise. The theory for low-rank matrix recovery under Poisson noise has been less developed. Moreover, the Poisson problems are quite different from their Gaussian counterpart, since under Poisson noise the variance of the noisy observations is proportional to the signal intensity. Moreover, instead of using $\ell_2$ error for data fit, we need to use a highly non-linear likelihood function.

Recently there has also been work that consider the more general noise models, including noisy 1-bit observations \cite{davenport20121}, which may be viewed as a case where the observations are Bernoulli random variables whose parameters depend on a underlying low-rank matrix; \cite{soni2014noisy, soniestimation}  consider the case where {\it all} entries of the low-rank matrix are observed and the observations are Poisson counts of the entries of the underlying matrix, and an upper bound is established (without a lower bound). In the compressed sensing literature, there is a line of research for sparse signal recovery in the presence of Poisson noise \cite{raginsky2010compressed, raginsky2011performance,jiang2014minimax} and the corresponding performance bounds. The recently developed SCOPT \cite{SCOPT13, SCOPT_journal} algorithm can also be used to solve the Poisson compressed sensing of sparse signals but may not be directly applied for Poisson matrix recovery.

In this paper, we extend the theory of low-rank matrix recovery to two related problems with Poisson observations: matrix recovery from compressive measurements, and matrix completion from observations of a subset of its entries. The matrix recovery problem from compressive measurements is formulated as a regularized maximum likelihood estimator with Poisson likelihood. We establish performance bounds by combining techniques for recovering sparse signals under Poisson noise \cite{raginsky2010compressed} and for establishing bounds in the case of low-rank matrices \cite{PlanThesis2011, candes2011tight}. Our results demonstrate that as the intensity of the signal increases, the upper bound on the normalized error decays at certain rate depending how well the matrix can be approximated by a low-rank matrix.

The matrix completion problem from partial observations is formulated as a maximum likelihood problem with proper constraints on the matrix $M$ (nuclear norm bound $\|M\|_* \leq \alpha\sqrt{r d_1 d_2}$ for some constant $\alpha$ and bounded entries $\beta \leq M_{ij}  \leq\alpha$)\footnote{Note that the formulation differs from the one-bit matrix completion case in that we also require a lower bound on each entry of the matrix. This is consistent with an intuition that the value of each entry can be viewed as the signal-to-noise ratio (SNR) for a Poisson observation, and hence this essentially poses a requirement for the minimum SNR.}.
We also establish upper and lower bounds on the recovery error, by adapting the arguments used for one-bit matrix completion \cite{davenport20121}. The upper and lower bounds nearly match up to a factor on the order of $\mathcal{O}(\log(d_1 d_2))$, which shows that the convex relaxation formulation for Poisson matrix completion is nearly optimal. We conjecture that such a gap is inherent to the Poisson problem in the sense that \yx{it may not be an artifact due to our proof techniques.}
Moreover, we also highlight a few important distinctions of Poisson matrix completion compared to the prior work on matrix completion in the absence of noise and with Gaussian noise:  (1) Although our arguments are adapted from one-bit matrix completion (where the upper and lower bounds nearly match), in the Poisson case there will be a gap between the upper and lower bounds, possibly due to the fact that Poisson distribution is only locally sub-Gaussian. In our proof, we notice that the arguments based on bounding all moments of the observations, which usually generate tight bounds for prior results with sub-Gaussian observations, do not generate tight bounds here; (2) We will need a lower bound on each matrix entry in the maximum likelihood formulation, which can be viewed as a requirement for the lowest signal-to-noise ratio (since the signal-to-noise ratio (SNR) of a Poisson observation with intensity $I$ is $\sqrt{I}$).

We also present a set of efficient algorithms, which can be used for both matrix recovery based on compressive measurements or based on partial observations. These algorithms include two generic (gradient decent based) algorithms: the proximal and accelerated proximal gradient descent methods, and an algorithm tailored to Poisson problems called the Penalized Maximum Likelihood Singular Value Threshold (PMLSVT) method. PMLSVT is derived by expanding the likelihood function locally in each iteration, and finding an exact solution to the local approximation problem which results in a simple  singular value thresholding procedure \cite{cai2010singular}. The performance of the two generic algorithms are analyzed theoretically. PMLSVT  is related to \cite{ji2009accelerated, wainwright2014structured, agarwal2010fast} and can be viewed as a special case where a simple closed form solution for the algorithm exists. Good performance of the PMLSVT is demonstrated with synthetic and real data including solar flare images and bike sharing count data. We show that the PMLSVT method has much lower complexity than solving the problem directly via semidefinite program and it has fairly good accuracy.

While working on this paper we realize a parallel work \cite{lafond2015low} which also studies performance bounds for low rank matrix completion with exponential family noise and using a different approach for proof (Poisson noise is a special case of theirs). Their upper bound for the mean square error (MSE) is on the order of $\mathcal{O}\left(\log(d_1 + d_2) r\max\{d_1, d_2\}/m\right)$ (our upper bound is $\mathcal{O}\left(\log(d_1 d_2)[r(d_1+d_2)/m]^{1/2}\right)$), and their lower bound is on the order of $\mathcal{O}\left(r\max\{d_1, d_2\}/m\right)$ (versus our lower bound is $\mathcal{O}\left([r(d_1+d_2)/m]^{1/2}\right)$. There might be two reasons for the difference. First, our sampling model (consistent with one bit matrix completion in \cite{davenport20121}) assumes {\it sampling without replacement}; 
therefore are at most $d_1 d_2$ observations, and each entry may be observed at most once. In contrast, \cite{lafond2015low} assumes {\it sampling with replacement}; therefore there can be multiple observations for the same entry. 
Since our result heavily depends on the sampling model, we suspect this may be a main reason for the difference. Another possible reason could be due to different formulations. The formulation for matrix completion in our paper is a constrained optimization with an exact {\it upper bound on the matrix nuclear norm}, whereas \cite{lafond2015low} uses a regularized optimization with a regularization parameter $\lambda$ (which is indirectly related to the nuclear norm of the solution), but there is no direct control of the matrix nuclear norm. Also, note that their upper and lower bounds also has a gap on the order of $\log (d_1 + d_2)$, which is consistent with our result.
%
%
On the other hand, compared with the more general framework for $M$-estimator \cite{WainwrightReview2014}, our results are specific to the Poisson case, which may possibly be stronger but do not apply generally.

The rest of the paper is organized as follows. Section \ref{sec:model} sets up the formalism for Poisson matrix completion.  Section \ref{sec:method_bound} presents matrix recovery based on constrained maximum likelihood and establishes the upper and lower bounds for the recovery accuracy. Section \ref{sec:algorithm} presents the PMLSVT algorithm that solves the maximum likelihood approximately and demonstrates its performance on recovering solar flare images and bike sharing count data. 
All proofs are delegated to the appendix.

The notation in this paper is standard. In particular, $\mathbb{R}_+$ denotes the set of positive real numbers and $\mathbb{Z}_+^m$ denotes a $m$-dimensional vector with positive integer entries; $\llbracket d \rrbracket =\{1,2,\ldots,d\}$; $(x)^+ = \max\{x,0\}$ for any scalar $x$; Let $[x]_j$ denote the $j$th element of a vector $x$; $\mathbb{I}\{[\varepsilon]\}$  is the indicator function for an event $\varepsilon$; $|A|$ denotes the number of elements in a set $A$; $\mbox{diag}\{x\}$ denotes a diagonal matrix with entries of a vector $x$ being its diagonal entries; $\textbf{1}_{d_1 \times d_2}$ denotes an $d_1$-by-$d_2$ matrix of all ones. Let $\|x\|_1$, $\|x\|_2$ denote the $\ell_1$ and $\ell_2$ norms of a vector $x$.
Let entries of a matrix $X$ be denoted by $X_{ij}$ or $[X]_{ij}$. For a matrix $X = [x_1, \ldots, x_n]$ with $x_j$ being the $j$th column, $\myvec(X) = [x_1\transpose, \ldots, x_n\transpose]\transpose$ denote vectorized matrix. 
Let $\|X\|$ be the spectral norm which is the largest absolute singular value, $\|X\|_{F} = (\sum_{i,j} X_{ij}^2)^{1/2}$ be the Frobenius norm, $\|X\|_*$ be the nuclear norm which is the sum of the singular values, $\|X\|_{1, 1} = \sum_{i}\sum_j |X_{ij}|$ be the $\ell_{1}$ norm, and finally $\|X\|_{\infty}$ = $\max_{ij}|X_{ij}|$ be the infinity norm of the matrix. Let $\mbox{rank}(X)$ denote the rank of a matrix $X$.
We say that a random variable $Z$ follows the Poisson distribution with a parameter $\lambda$ (or $Z \sim \mbox{Poisson}(\lambda))$ if its probability mass function $\mathbb{P}(Z=k) = e^{-\lambda}\lambda^k/(k!)$). Finally, let $\mathbb{E}[Z]$ denote the expectation of a random variable $Z$.
%

The only set of non-conventional notation that we use is the following. By a slight abuse of notation, we denote the Kullback-Leibler (KL) divergence between two Poisson distributions with parameters $p$ and $q$, $p,q \in \mathbb{R}_+$ as
\[
D(p\|q) \triangleq p\log(p/q) - (p-q),
\]
and denote the Hellinger distance between two Poisson distributions with parameters $p$ and $q$  as
\[
d_H^2(p, q) \triangleq 2-2\exp\left\{-\frac{1}{2}\left(\sqrt{p}-\sqrt{q}\right)^2\right\}.
\]
It should be understood that the KL distance and the Hellinger distance are defined between two distributions and here the arguments $p$ and $q$ are merely parameters of the Poisson distributions since we restrict our attention to Poisson.
Based on this, we also denote, by a slight abuse of notation, the average KL and Hellinger distances for two sets of Poisson distributions whose parameters are determined by entries of two matrices $P$, $Q \in \mathbb{R}_+^{d_1 \times d_2}$:
$$
D(P\|Q) \triangleq \frac{1}{d_1 d_2}\sum_{i,j}D(P_{ij}\|Q_{ij}),
$$
$$
d_H^2(P,Q) \triangleq \frac{1}{d_1 d_2}\sum_{i,j}d_H^2(P_{ij},Q_{ij}).
$$

\section{Formulation}
\label{sec:model}


\subsection{Matrix recovery}\label{sec:matrix_recovery}

Given a matrix $M \in \mathbb{R}_+^{d_1\times d_2}$ consisting of positive entries, we obtain $m$ Poisson measurements $y = [y_1, \ldots, y_m]^\intercal \in \mathbb{Z}_+^m$ that take the forms of
\begin{equation}
y_i \sim \mbox{Poisson} (\mbox{tr}(A_i\transpose M)), \quad i = 1, \ldots, m.
\label{obs}
\end{equation}
where $A_i \in \mathbb{R}^{d_1\times d_2}$, and it models the measuring process of physical devices, e.g., the compressive imaging system given in \cite{brady2009optical}. \yx{For example, $A_i$ can be interpreted as a mask imposed on the scene, and one measurement is made by integrating the total photon intensity passing through the mask. }
Our goal is to estimate the signal $M\in \mathbb{R}_+^{d_1\times d_2}$ from measurements $y\in \mathbb{Z}_+^m$. Define a linear operator $\mathcal{A}: \mathbb{R}_+^{d_1\times d_2} \rightarrow \mathbb{R}^{m}$ such that
\ben
[\cA M]_i \triangleq \tr(A_i\transpose M).
\een
So by defining
\[
A \triangleq \begin{bmatrix}
\mbox{vec}(A_1)\transpose \\
\vdots\\
\mbox{vec}(A_m)\transpose
\end{bmatrix}, \quad f \triangleq \myvec(M),
\]
we can write
\begin{equation}
\cA M = A f.
\label{meaprocess}
\end{equation}

The following assumptions are made for matrix recovery. First, the total intensity of $M$ given by
\[
I \triangleq \|M\|_{1, 1}
\]
is known a priori. Second, assume an entry-wise lower bound $[M]_{jk} \geq c$ for some constant $c > 0$. \yx{This prevents the degeneracy later on since the rate of a Poisson random variable has to be positive. }
\yx{Third, motivated by the assumptions made in Poisson compressed sensing \cite{raginsky2010compressed} and to have physically realizable optical systems}, we assume that $\mathcal{A}$ satisfies the following constraints:
(1) positivity-preserving: \yx{for any nonnegative input matrix $M$, the measurements must also be nonnegative; equivalently,}
\[M_{ij} \geq 0 ~ \mbox{for all} ~ i, j \Rightarrow [\cA M]_i \geq 0 ~\mbox{for all} ~i;\] (2) flux-preserving: \yx{the mean total intensity of the observed signal must not exceed to total intensity incident upon the system:}
\[\sum_{i=1}^m [\cA M]_i \leq \|M\|_{1, 1}.\]
\yx{Physically, this means the photon counting measurements can only be positive and we cannot measure more light than what is available.}

We consider a regularized maximum-likelihood formulation. In the matrix recovery problem, the log-likelihood function is given by
\begin{equation}
L_{\mathcal{A}, y}(X) \triangleq \sum_{i=1}^m \{y_i \log [\cA X]_i - [\cA X]_i\} -\lambda \pen(X), \label{L}
\end{equation}
where the subscript $\cA$ and $y$ indicate the given data in defining the log likelihood function.
Based on the previous assumptions, we define a countable candidate set
\ben
\Gamma \triangleq  \{X_i \in \mathbb{R}_+^{d_1\times d_2}: \|X_i\|_{1, 1} = I, [X_i]_{jk} \geq c,  i = 1, 2, \ldots\},
\label{gamma_set_def}
\een
for some constant $c > 0$. This $\Gamma$ can be interpreted as a discretized feasible domain of the general problem. Note that the true matrix $M \in \Gamma$. Also introduce a regularization function that satisfies the Kraft inequality \cite{raginsky2010compressed},\cite{raginsky2011performance},
\begin{equation}
\sum_{X \in \Gamma} e^{-\pen(X)} \leq 1. \label{Kraft_inequatliy}
\end{equation}
A Kraft-compliant regularization function typically assigns a small value for a lower rank matrix $X$ and assigns a large value for a higher rank matrix $X$. Using Kraft-compliant regularization to prefix codes for estimators is a commonly used technique in constructing estimators \cite{raginsky2010compressed}.
An estimator $\widehat{M}$ is obtained by solving the following convex optimization problem:
\begin{equation}
\widehat{M} = \underset{X \in \Gamma}{\arg\max}~ L_{\cA, y}(X). \quad \{\mbox{matrix recovery}\}
\label{estimator}
\end{equation}

\begin{remark}[Relation to Poisson compressed sensing] In Poisson compressed sensing \cite{raginsky2011performance}, the measurement vector is given by $y \sim \mbox{Poisson}(A\transpose x)$, where $A$ is a sensing matrix and $x$ is a sparse vector of interest. From (\ref{meaprocess}), we see that the measurement model of Poisson matrix recovery can be written in this form too. However, if we vectorize $M$ and solve it as a Poisson compressed sensing problem (\ref{meaprocess}), the low-rank structure of $M$ will be lost. Hence, the extension to Poisson matrix recovery is important, since in many applications such as compressive imaging the signal is well modeled as a nearly low-rank matrix.

\end{remark}

\subsection{Matrix completion}
A related problem is matrix completion. Given a matrix $M \in \mathbb{R}_+^{d_1 \times d_2}$ consisting of positive entries, we obtain noisy observations for a subset of its entries on an index set $\Omega \subset \llbracket d_1\rrbracket \times \llbracket d_2\rrbracket$. The indices are randomly selected with $\mathbb{E}[|\Omega|]=m$. In other words, $\mathbb{I}\{(i,j) \in \Omega\}$ are i.i.d. Bernoulli random variables with parameter $m/(d_1 d_2)$.
The observations are Poisson counts of the observed matrix entries and they are mutually independent
\begin{equation}
Y_{ij} \sim  \mbox{Poisson}(M_{ij}), \quad \forall (i,j) \in \Omega.
\label{poissonmodel}
\end{equation}
Our goal is to recover $M$ from the Poisson observations $\{Y_{ij}\}_{(i,j) \in \Omega}$.

The following assumptions are made for the matrix completion problem. First, we set an upper bound $\alpha>0$ for each entry $M_{ij} \leq \alpha$ to entail that the recovery problem is  well-posed \cite{negahban2012restricted}. This assumption is also reasonable in practice; for instance, $M$ may represent an image which is usually not too spiky. The second assumption is characteristic to Poisson matrix completion: we set a lower bound $\beta>0$ for each entry $M_{ij} \geq \beta$. This entry-wise lower bound is required by our later analysis (so that the cost function is Lipschitz), and it also has an interpretation of a minimum required signal-to-noise ratio (SNR), since SNR of a Poisson observation with intensity $I$ is given by $\sqrt{I}$. \yx{Third, we make a similar assumption to  one-bit matrix completion \cite{davenport20121}; the nuclear norm of $M$ is upper bounded $\|M\|_* \leq \alpha \sqrt{rd_1 d_2}$. This is a relaxation of the assumption that $M$ has a rank exactly $r$ (some small integer).} This particular choice arises from the following consideration. If $M_{ij} \leq \alpha$ and $\mbox{rank}(M) \leq r$, then
\[\|M\|_* \leq \sqrt{r}\|M\|_F \leq \sqrt{rd_1 d_2}\|M\|_{\infty} \leq \alpha \sqrt{r d_1 d_2}.\]

We consider a formulation by maximizing the log-likelihood function of the optimization variable $X$ given our observations subject to a set of convex constraints. In the matrix completion problem, the log-likelihood function is given by
\begin{equation}
    F_{\Omega, Y}(X) = \sum_{(i,j)\in \Omega} Y_{ij}\log X_{ij} - X_{ij},
\label{likelihood}
\end{equation}
where the subscript $\Omega$ and $Y$ indicate the data involved in the maximum likelihood function $F$. Based on previous assumptions, we define a candidate set
\begin{equation}
\begin{split}
&\mathcal{S} \triangleq \left\{ X \in \mathbb{R}_+^{d_1 \times d_2} : \|X\|_* \leq \alpha \sqrt{r d_1 d_2}, \right. \\
&\qquad  \qquad \left. \beta \leq X_{ij} \leq \alpha, \forall (i,j) \in \llbracket d_1\rrbracket \times \llbracket d_2\rrbracket \right\}.
\end{split}
\label{searchspace}
\end{equation}
An estimator $\widehat{M}$ can be obtained by solving the following convex optimization problem:
\begin{equation}
\widehat{M} = \underset{X\in \mathcal{S}}{\arg\max} ~F_{\Omega,Y}(X). \quad \{\mbox{matrix completion}\}
\label{optimization_problem}
\end{equation}

\subsection{Relation of two formulations}

\yx{Note that the matrix completion problem can also be formulated as a regularized maximum likelihood function problem similar to (\ref{estimator}). However, we  consider the current formulation for the convenience of drawing connections, respectively, between Poisson matrix recovery  and Poisson compressed sensing studied in\cite{raginsky2010compressed}, as well as Poisson matrix completion and one-bit matrix completion studied in \cite{davenport20121}.

Indeed, these two formulations in the forms of (\ref{estimator}) and (\ref{optimization_problem}) are related by the well-known duality theory in optimization (see, e.g. \cite{ConvexOpt}). Consider a penalized convex optimization problem:
\begin{equation}
 \quad \min_{x} f(x)+\lambda g(x), \quad \lambda \geq 0, \label{P1}
\end{equation}
and the constrained convex optimization problem:
\begin{equation}
 \quad \min_{x} f(x) \quad \mbox{subject to} \quad g(x)\leq c. \label{P2}
\end{equation}
Denote $x^*$ as the solution to (\ref{P1}). Then $x^*$ is also the solution to (\ref{P2}), if we set $c=g(x^*)$. Conversely, denote $x^*$ as the solution to problem (\ref{P2}). We can interpret $\lambda \geq 0$ as the the Lagrange multiplier and consider the Lagrange dual problem. Under Slater's condition (i.e. there exists at least one $x$ such that $g(x)<c$),  there is at least one $\lambda$ such that $x^*$ is also the solution to problem (\ref{P1}). Therefore, (\ref{P1}) and (\ref{P2}) are equivalent in the sense that the two problems have the same minimizer for properly chosen parameters. More details can be found in \cite{loft2009efficient}. Using the suggestion by Theorem 1 in \cite{loft2009efficient}, we choose $\lambda$ in (\ref{estimator}) to be around $1/\alpha \sqrt{r d_1 d_2}$.}

\section{Performance Bounds}
\label{sec:method_bound}

In the following, we use the squared error
\ben
R(M, \widehat{M}) \triangleq \|M-\widehat{M}\|_F^2,
\een
as a performance metric for both matrix recovery and matrix completion problems.

\subsection{Matrix recovery}

\yx{To extend the upper bounds in Poisson compressed sensing \cite{raginsky2010compressed} to the matrix recovery setting, we introduce a class of nearly low-rank matrices whose singular values decay geometrically. Using a particular choice for the sensing matrices $A_i$ that satisfy certain property (Lemma \ref{thm_RIP}), we present a performance guarantee for the estimator in the general setting (Lemma \ref{thm_general}) and then for nearly low-rank matrices (Theorem \ref{regret_for_nearly_low_rank}). }

\subsubsection{Sensing operator}
Let  $Z_i$, $i = 1, \ldots, m$ denote a $d_1$-by-$d_2$ matrix with entries i.i.d. follow the distribution
\ben
[Z_i]_{jk} = \left\{\begin{array}{ll}
-\left(\frac{1-p}{p}\right)^{1/2}, & \hbox{with probability } p;\\
\left(\frac{p}{1-p}\right)^{1/2}, & \hbox{with probability } 1-p.
\end{array}\right.
\een
Define
\[
\tilde{A}_i \triangleq Z_i/\sqrt{m},
\]
which consists of a random part and a deterministic part:
\ben
A_i \triangleq \left[\frac{p(1-p)}{m}\right]^{1/2} \tilde{A}_i + \frac{1-p}{m} \vect{1}_{d_1 \times d_2}.
\een
\yx{This construction is inspired by \cite{raginsky2010compressed} for the purpose of corresponding to a feasible physical system. In particular, every element of $A_i$ is nonnegative and scaled properly, so that the measured photon intensity is no greater than the photon intensity of the original signal. }
It can be verified that $A_i$ and the associated $\cA$ satisfy the requirements in the previous section. In particular, (1) all entries of $A_i$ take values of 0 or $1/m$;
(2) $\cA$ satisfies flux preserving: for any matrix $X$ with positive entries $[X]_{ij} > 0$, since all entries of $A_i$ are less than $1/m$,
\[
\|\cA X\|_{1} = \sum_{i=1}^m \sum_{j=1}^{d_1} \sum_{k=1}^{d_2} [A_i]_{jk} [X]_{jk} \leq  \sum_{j=1}^{d_1} \sum_{k=1}^{d_2} [X]_{jk} = I;
\]
(3) with probability at least $1-mp^{d_1 d_2}$, every matrix $A_i$ has at least one non-zero entry. It follows that for a matrix $X$ such that $[X]_{ij} \geq c$, \yx{under the event described above}, we have
\begin{equation}
[\cA X]_i = \sum_{j=1}^{d_1} \sum_{k=1}^{d_2} [A_i]_{jk} [X]_{jk} \geq c\sum_{j=1}^{d_1} \sum_{k=1}^{d_2} [A_i]_{jk}  \geq \frac{c}{m}. \label{lower_measurement_bound}
\end{equation}
This prevents degeneracy since $[\cA X]_i$ is a rate of Poisson and has to be positive.  

\yx{The random part of the operator $\cA$ has certain restrictive isometry property (RIP) similar to \cite{CandesTao2006}, as formalized in the following lemma and proved by extending Theorem 1 of\cite{raginsky2010compressed}. Similar to \cite{raginsky2010compressed}, Lemma \ref{thm_RIP} is used in proving the performance upper bound in a general setting (Lemma \ref{thm_general}).}
\begin{lemma}[RIP of operator $\mathcal{A}$] \label{thm_RIP}
Consider the  operator $\tilde{\mathcal{A}}$ defined by $\tilde{\mathcal{A}} X \triangleq [\mbox{tr}(\tilde{A}_1 X), \ldots, \mbox{tr}(\tilde{A}_m X)]\transpose \in \mathbb{R}^m$.
For all $X_1, X_2 \in \mathcal{B}^{d_1\times d_2}$ where $\mathcal{B}^{d_1\times d_2} \triangleq \{X \in \mathbb{R}^{d_1\times d_2}: \|X\|_{1, 1} = 1\}$,
there exist absolute constants $c_1, c_2 > 0$ such that the bound
\[
\|X_1 - X_2\|_F^2 \leq 4\|\tilde{\mathcal{A}}X_1 - \tilde{\mathcal{A}}X_2\|_2^2 + \frac{2c_2^2 \xi_p^4 \log(c_2 \xi_p^4 d_1 d_2/m)}{m}
\]
holds with probability at least $1-e^{-c_1 m/\xi_p^4}$, where
\begin{equation}
\xi_p\triangleq \left\{
\begin{array}{ll}
\left[\frac{3}{2p(1-p)}\right]^{1/2}, & \hbox{if } p\neq 1/2; \\
1, & \hbox{otherwise.}
\end{array}
\right.\label{def_gamma_p}
\end{equation}
Moreover, there exist absolute constants $c_3, c_4 > 0$ such that for any finite set $\mathcal{T}\subset \mathcal{S}^{d_1\times d_2 - 1}$, where $\mathcal{S}^{d_1\times d_2 - 1} \triangleq \{X \in \mathbb{R}^{d_1\times d_2}: \|X\|_F = 1\}$ is the unit sphere,
if $m \geq c_4 \xi_p^4 \log_2|\mathcal{T}|$, then
\[
\frac{1}{2} \leq \|\tilde{\mathcal{A}} X\|_2^2 \leq \frac{3}{2}, \quad \mbox{for all } X \in \mathcal{T}
\]
holds with probability at least $1- e^{-c_3 m/\xi_p^4}$.
\end{lemma}
This lemma follows directly from applying Theorem 1 in \cite{raginsky2010compressed} to the vectorized version of the matrix recovery problem (\ref{meaprocess}) as well as using the fact $\|X\|_F = \|\mbox{vec}(X)\|_2$ and $\|X\|_{1,1} = \|\mbox{vec}(X)\|_1$ for a matrix $X$.

\subsubsection{General matrices}
For an arbitrary matrix $M$, given a suitable candidate set $\Gamma$, and if the operator $\cA$ satisfies the RIP in Lemma \ref{thm_RIP}, we may obtain a regret bound for the estimator $\widehat{M}$ obtained from (\ref{estimator}). Note that the result does not require $M$ to be nearly low-rank. \yx{The result takes a similar form as Theorem 2 of\cite{raginsky2010compressed} except that the vector signal dimension is replaced by $d_1 d_2$. Lemma \ref{thm_general} is used for establishing the regret bound  for nearly low-rank matrices in Theorem \ref{regret_for_nearly_low_rank}. }

\begin{lemma}[Regret bound] \label{thm_general}
Assume the candidate set $\Gamma$ defined in (\ref{gamma_set_def}). 
Let $\mathcal{G}$ be the collection of all subsets $\Gamma_0 \subseteq \Gamma$, such that $|\Gamma_0| \leq 2^{m/(c_4\xi_p^4)}$ for $\xi_p$ defined in (\ref{def_gamma_p}). Then with probability at least $1-d_1 d_2 e^{-km}$ for some positive constant $k$ depending on $c_1$, $c_3$ and $p$:
\begin{equation}
\begin{split}
& \frac{1}{I^2}\mathbb{E}[R(M, \widehat{M})] \leq C_{m, p} \min_{\Gamma_0 \in \mathcal{G}} \min_{\widetilde{M} \in \Gamma_0} \left[\frac{R(M, \widetilde{M})}{I^2} + \lambda\frac{\pen(\widetilde{M})}{I}\right] \\
& \quad  + \frac{2c_2^2 \xi_p^4 \log (c_2 \xi_p^4 d_1 d_2/m)}{m},
\end{split}\label{thm2_result}
\end{equation}
where $C_{m,p} \triangleq \max\left\{\frac{24}{c}, \frac{16}{p(1-p)} \right\}m,$ and the expectation is taken with respect to the random Poisson measurements $y_i \sim \mbox{Poisson}(\mbox{tr}(A_i M))$ for a fixed $\tilde{\mathcal{A}}$.
\end{lemma}
\begin{remark}
\yx{This bound can be viewed as an ``oracle inequality'': if given the knowledge of the true matrix $M$, we may first obtain an ``oracle error'', the term inside the square bracket in (\ref{thm2_result}), which is the error associated with the best approximation of $M$ in the set $\Gamma_0$, for all such possible sets $\Gamma_0\subseteq\Gamma$  that the size of $\Gamma_0$ is at most $\mathcal{O}(2^m)$. Then the normalized expected squared error of the maximum likelihood estimator is within a constant factor from this oracle error.}
\end{remark}

\subsubsection{Nearly low-rank matrices}

\yx{In the following, we establish the regret bound for nearly low-rank matrices, whose singular values decay geometrically. Theorem \ref{regret_for_nearly_low_rank} is obtained by extending Theorem 3 in \cite{raginsky2010compressed}, with key steps including realizing that the vector formed by the singular values of the nearly low-rank matrix is ``compressible'', and invoking a covering number for certain subset of low-rank matrices in \cite{PlanThesis2011}.}

\yx{To extend the definition of compressible signals \cite{CandesTao2006, raginsky2010compressed} in the matrix setting, we consider a family of nearly low-rank matrices in the following sense. }
Assume the singular value decomposition of a matrix $X \in \mathbb{R}^{d_1\times d_2}$ is given by $X = U \mbox{diag}\{\theta\} V$, where $\theta = [\theta_1, \ldots, \theta_d]\transpose$ is a vector consists of singular values, and $|\theta_1| \geq |\theta_2| \geq \cdots \geq |\theta_d|$ with \[d \triangleq \min\{d_1, d_2\}.\]
Suppose there exists $0 < q < \infty$, \yx{$0 < \varrho < \infty$}, such that
\begin{equation}
|\theta_j|\leq \varrho I j^{-1/q}, \quad j = 1, \ldots, d. \label{geo_decay}
\end{equation}
Any $\theta$ satisfying (\ref{geo_decay}) is said to belong to the weak-$\ell_q$ ball of radius $\varrho I$.
\yx{It can be verified that the condition (\ref{geo_decay}) is equivalent to the following
\[
|\{j \in \llbracket d \rrbracket: |\theta_j|\geq cI\}| \leq \left(\varrho/c\right)^q
\]
holding for all $c > 0$. Hence, if we truncate the singular values of a nearly low-rank matrix less than $cI$, then the rank of the resulted matrix is $(\varrho/c)^q$. In other words, $q$ controls the speed of the decay: the smaller the $q$, the faster the decay, and the smaller the rank of the approximating matrix by thresholding the small singular values. We shall focus primarily on the case $0 < q < 1$.  Also, for $\|X\|_{1, 1} = I$,
\[
|\theta_j| \leq \left(\sum_{j=1}^d \theta_j^2\right)^{1/2} = \|X\|_F \leq \|X\|_{1, 1} = I,
\]
so we can take $\varrho$ to be a constant independent of $I$ or $d$. That is also the reason in (\ref{geo_decay}) for us to choose the constant to be $\varrho I$ since $\varrho$ has the meaning of a factor relatively to the total intensity $I$.
}

For these nearly low-rank matrices, the best rank-$\ell$ approximation to $X$ can be constructed as
\[X^{(\ell)} \triangleq U \diag\{\theta^{(\ell)}\} {V}\transpose,\]
where $\theta^{(\ell)} = [\theta_1, \ldots, \theta_\ell, 0, \ldots, 0]\transpose$.
Note that 
\begin{equation}
\begin{split}
\|X - X^{(\ell)}\|_F^2 &= \|U \diag\{\theta - \theta^{(\ell)}\} V\transpose \|_F^2 \\
&=  \|\theta - \theta^{(\ell)}\|_2^2 \leq I^2 c_0 \varrho^2 \ell^{-2\alpha'},
\end{split}
\label{approx_inq}
\end{equation}
for some constant $c_0 > 0$ that depends only on $q$, and  $\alpha'=(1/q-1/2)$ \cite{CandesTao2006}. \yx{For matrix to be nearly low-rank, we want $q$ to be close to 0, and hence usually $\alpha' = 1/q- 1/2 >0$.} The following regret bound for a nearly low-rank matrix is a consequence of Lemma \ref{thm_general}.

\begin{theorem}[Matrix recovery; regret bound for nearly low-rank matrices]\label{regret_for_nearly_low_rank}
Assume $M \in \mathbb{R}^{d_1\times d_2}_+$ is nearly low-rank, $M_{ij} \geq c$ for some positive constant $c\in(0, 1)$. Then there exists a finite set of candidates $\Gamma$, and a regularization function satisfying the Kraft inequality (\ref{Kraft_inequatliy}) such that the bound
\begin{equation}
\begin{split}
&\frac{1}{I^2}\mathbb{E}[R(M, \widehat{M})] \leq \\
&\mathcal{O}(m)\left\{2+ \frac{164}{d} + \min_{1\leq \ell \leq \ell_*} \left[c_0 \varrho^2 \ell^{-2\alpha'} + \frac{\lambda \ell (d_1+d_2+4)  \log_2 d}{2I}\right]\right\} \\
&+ \mathcal{O}\left(\frac{\log(d_1d_2/m)}{m}\right),
\end{split}
\label{bound_MSE}
\end{equation}
holds with the same probability as in Lemma \ref{thm_general} for some constant $c_0 > 0$ that depends only on $q$, 
\begin{equation}
\ell_* \triangleq 2m/[c_4 \xi_p^4(d_1 + d_2 + 4) \log_2 d],
\end{equation}
 and $\xi_p$ is defined in (\ref{def_gamma_p}). Here the expectation is with respect to the random Poisson measurements for a fixed realization of $\cA$.
\label{theoremforcompressive}
\end{theorem}
\begin{remark}
\yx{In fact we may compute the minimum in the optimization problem inside Theorem \ref{regret_for_nearly_low_rank} in terms of $\ell$. Note that the first term $\ell^{-2\alpha'}$ is decreasing in $\ell$, and the second term $[\lambda \ell (d_1 + d_2 + 4)  \log_2 d ]/(2I)$ is increasing in $\ell$, so we may readily solve that the minimum in (\ref{bound_MSE}) is obtained by
\[
\ell_{\rm min} = \min\left\{\left[
\frac{4c_0 \varrho^2 \alpha' I}{\lambda (d_1 + d_2 + 4)\log_2 d}
\right]^{\frac{1}{2\alpha' + 1}}, ~\ell_*\right\}.
\]
If further $\ell_* \geq \{4\alpha'I/[\lambda(d_1 + d_2 +4)\log_2 d]\}^{1/(2\alpha'+1)}$, which is true when the number of measurements $m$ is sufficiently large, we may substitute $\ell_{\rm min}$ into the expression and simplify the upper bound (\ref{bound_MSE}) to be
\begin{equation}
\begin{split}
\mathcal{O}(m) \left(
\frac{\lambda (d_1 + d_2 + 4)\log_2 d}{4c_0 \varrho^2 I}
\right)^{\frac{2\alpha'}{2\alpha'+1}} + \mathcal{O}\left(\frac{\log(d_1d_2/m)}{m}\right)
\end{split}\nonumber.
\end{equation}
The implication of Theorem \ref{regret_for_nearly_low_rank} is that, for a nearly low-rank matrix whose singular values decay geometrically at a rate of $j^{-1/q}$, $j = 1, 2, \ldots$ and the total intensity of the matrix is $I$,
\begin{equation}
\begin{split}
&\mbox{reconstruction error} \\
&\propto m \left[\frac{(d_1 + d_2)\log_2 \min\{d_1, d_2\}}{I}\right] + \frac{\log (d_1 d_2/m)}{m},
\end{split}\nonumber
\end{equation}
for $m$ sufficiently large. This also implies that the intensity $I$ (or SNR $\sqrt{I}$) needs to be sufficiently large for the error bound to be controllable by increasing the number of measurements.
}
\end{remark}


\subsection{Matrix completion}

For matrix completion, we first establish an upper bound for estimator in (\ref{optimization_problem}), and then present an information theoretic lower bound which nearly matches the upper bound up to a logarithmic factor $\mathcal{O}(d_1 d_2)$.



\begin{theorem}[Matrix completion; upper bound]
\label{maintheorem}
    Assume $M \in \mathcal{S}$, $\Omega$ is chosen at random following our Bernoulli sampling model with $\mathbb{E}[|\Omega|] = m$, and $\widehat{M}$ is the solution to (\ref{optimization_problem}). Then with a probability exceeding $\left(1-C/(d_1 d_2)\right)$, we have
    \begin{equation}
    \begin{split}
       & \frac{1}{d_1 d_2}R(M, \widehat{M}) \leq C' \left(\frac{8\alpha T}{1-e^{-T}}\right)\cdot (\frac{\alpha \sqrt{r}}{\beta})
       \cdot \\
       &  \left( \alpha(e^2-2) + 3\log(d_1 d_2) \right) \cdot \left(\frac{d_1 +d_2}{m}\right)^{1/2} \cdot\\
       &  \left[1+\frac{(d_1+d_2)\log(d_1 d_2)}{m}\right]^{1/2}.
    \end{split}
    \label{bound:MC}
    \end{equation}
    If $m\geq (d_1+d_2)\log(d_1 d_2)$, then (\ref{bound:MC}) simplifies to
    \begin{equation}
    \begin{split}
    &\frac{1}{d_1 d_2}R(M, \widehat{M}) \leq \sqrt{2}C' \left(\frac{8\alpha T}{1-e^{-T}}\right) \cdot \left(\frac{\alpha \sqrt{r}}{\beta} \right) \cdot \\
    &\left( \alpha(e^2-2) + 3\log(d_1 d_2) \right) \cdot \left(\frac{d_1 + d_2}{m}\right)^{1/2}.
    \end{split}
    \label{bound:MC2}
    \end{equation}
    Above, $C', C$ are absolute constants and $T$ depends only on $\alpha$ and $\beta$.
\yx{Here the expectation and probability are with respect to the random Poisson observations and Bernoulli sampling model. }
\label{maintheorem}
\end{theorem}

The proof of Theorem \ref{maintheorem} is an extension of the ingenious arguments for one-bit matrix completion \cite{davenport20121}. The extension for Poisson case here is nontrivial for various aforementioned reasons (notably the non sub-Gaussian and only locally sub-Gaussian nature of the Poisson observations). An outline of our proof is as follows. First, we establish an upper bound for the Kullback-Leibler (KL)  divergence $D(M \| X)$ for any $X \in \mathcal{S}$ by applying Lemma \ref{firstlemma} given in the appendix. Second, we find an upper bound for the Hellinger distance $d_H^2(M, \widehat{M})$ using the fact that the KL divergence can be bounded from below by the Hellinger distance. Finally, we bound the mean squared error in Lemma \ref{secondlemma} via the Hellinger distance.
%

\begin{remark}
Fixing $m$, $\alpha$ and $\beta$, the upper bounds (\ref{bound:MC}) and (\ref{bound:MC2}) in Theorem \ref{maintheorem} increase as the upper bound on the nuclear norm increases, which is proportional to $\sqrt{r d_1 d_2}$. This is consistent with the intuition that our method is better at dealing with approximately low-rank matrices than with nearly full rank matrices. On the other hand, fixing $d_1, d_2, \alpha$, $\beta$ and $r$,  the upper bound decreases as $m$ increases, which is also consistent with the intuition that the recovery is more accurately with more observations.
\end{remark}

\begin{remark}
Fixing $\alpha$, $\beta$ and $r$, the upper bounds (\ref{bound:MC}) and (\ref{bound:MC2}) on the mean-square-error per entry can be arbitrarily small, in the sense that the they tend to zero as $d_1$ and $d_2$ go to infinity and the number of the measurements $m =\mathcal{O}((d_1+d_2)\log^\delta(d_1 d_2))$ ($m\leq d_1 d_2$) for $\delta >2$.
\end{remark}

We may obtain an upper bound on the KL divergence (which may reflect the true distribution error) as a consequence of Theorem \ref{maintheorem}.
\begin{corollary}[Upper bound for KL divergence]
 Assume $M \in \mathcal{S}$, $\Omega$ is chosen at random following the Bernoulli sampling model with $\mathbb{E}[|\Omega|] = m$, and $\widehat{M}$ is the solution to (\ref{optimization_problem}). Then with a probability exceeding $\left(1-C/(d_1 d_2)\right)$,
    \begin{equation}
    \begin{split}
    D(M\|\widehat{M}) \leq & 2C'\left( \alpha\sqrt{r}/\beta \right) \left( \alpha(e^2-2) + 3\log(d_1 d_2) \right) \cdot \\
    & \left(\frac{d_1 +d_2}{m}\right)^{1/2}\cdot \left[1+\frac{(d_1+d_2)\log(d_1 d_2)}{m}\right]^{1/2}.
    \label{corollaryuse1}
    \end{split}
    \end{equation}
Above, C and C' are absolute constants. Here the expectation and probability are with respect to the random Poisson observations and Bernoulli sampling model.
\end{corollary}

%
%


The following theorem establishes a lower bound and demonstrates that  there exists an $M \in \mathcal{S}$ such that {\it any} recovery method cannot achieve a mean square error per entry less than the order of $\mathcal{O}(\sqrt{r\max\{d_1, d_2\}/m})$.

\begin{theorem}[Matrix completion; lower bound]
Fix $\alpha$, $\beta$, $r$, $d_1$, and $d_2$ to be such that $\beta \geq 1$, $\alpha \geq 2\beta$, $d_1\geq 1$, $d_2 \geq 1$, $r \geq 4$, and $\alpha^2 r \max\{d_1, d_2\} \geq C_0$. Fix $\Omega_0$ be an arbitrary subset of $\llbracket  d_1\rrbracket \times \llbracket d_2\rrbracket$ with cardinality $m$. Consider any algorithm which, for any $M \in \mathcal{S}$, returns an estimator $\widehat{M}$. Then there exists $M \in \mathcal{S}$ such that with probability at least $3/4$,
\begin{equation}
\begin{split}
&\frac{1}{d_1 d_2} R(M, \widehat{M}) \\
&\geq \min\left\{\frac{1}{256}, C_2 \alpha^{3/2} \left[\frac{r\max\{d_1,d_2\}}{m}\right]^{1/2}\right\},
\end{split}
\label{lowerbound}
\end{equation}
as long as the right-hand side of (\ref{lowerbound}) exceeds $C_1 r\alpha^2 /\min\{d_1,d_2\}$, where $C_0$, $C_1$ and $C_2$ are absolute constants. Here the probability is with respect to the random Poisson observations only.
\label{maintheorem2}
\end{theorem}

Similar to \cite{davenport20121, candes2013well}, the proof of Theorem \ref{maintheorem2} relies on  information theoretic arguments outlined as follows. First we find a set of matrices $\chi \subset \mathcal{S}$ so that the distance between any $X^{(i)}, X^{(j)} \in \chi$, identified as $\|X^{(i)}- X^{(j)}\|_F$, is sufficiently large. Suppose we obtain measurements of a selected matrix in $\chi$ and recover it using an arbitrary method. Then we could determine which element of $\chi$ was chosen, if the recovered matrix is sufficiently close to the original one. However, there will be a lower bound on how close the recovered matrix can be to the original matrix, since due to Fano's inequality the probability of correctly identifying the chosen matrix is small.


\begin{remark}
Fixing $\alpha, \beta$ and $r$, the conditions in the statement of Theorem 3 can be satisfied if we choose sufficiently large $d_1$ and $d_2$. 
\end{remark}

\begin{remark}
When $m\geq (d_1+d_2)\log(d_1 d_2)$, the ratio between the upper bound in (\ref{bound:MC2}) and the lower bound in (\ref{lowerbound}) is on the order of $\mathcal{O}(\log(d_1 d_2))$.
Hence, the lower bound matches the upper bound up to a logarithmic factor.
\end{remark}


Our formulation and results for Poisson matrix completion are inspired by one-bit matrix completion \cite{davenport20121}, yet with several important distinctions. In one-bit matrix completion, the value of each observation $Y_{ij}$ is binary-valued and hence bounded; whereas in our problem, each observation is a Poisson random variable which is unbounded and, hence, the arguments involve bounding measurements have to be changed. In particular, we need to bound $\max_{ij} Y_{ij}$ when $Y_{ij}$ is a Poisson random variable with intensity $M_{ij}$. Moreover, the Poisson likelihood function is non Lipschitz (due to a bad point when $M_{ij}$ tends to zero), and hence we need to introduce a lower bound on each entry of the matrix $M_{ij}$, which can be interpreted as the lowest required SNR. Other distinctions also include analysis taking into account of the property of the Poisson likelihood function, and using KL divergence as well as Hellinger distance that are different from those for the Bernoulli random variable as used in \cite{davenport20121}.

\section{Algorithms}
\label{sec:algorithm}

In this section we develop efficient algorithms to solve the matrix recovery (\ref{estimator}) and matrix completion problems (\ref{optimization_problem}). In the following, we use nuclear norm regularization function for the matrix recovery in (\ref{estimator}). Then, (\ref{estimator}) and (\ref{optimization_problem})  are both semidefinite program (SDP), as they are nuclear norm minimization problems with convex feasible domains. Hence, we may solve it, for example, via the interior-point method \cite{liu2009interior}. Although the interior-point method may return an exact solution to (\ref{optimization_problem}), it does not scale well with the dimensions of the matrix $d_1$ and $d_2$ as the complexity of solving SDP is $O(d_1^3+d_1d_2^3+d_1^2 d_2^2)$.

We develop two set of algorithms that can solve both problems faster than the interior point methods. These algorithms including the generic gradient descent based methods, and a Penalized Maximum Likelihood Singular Value Threshold (PMLSVT) method tailored to our problem. We analyzed the performance of the generic methods. Although there is no theoretical performance guarantee, PMLSVT is computationally preferable under our assumptions. Another possible algorithm not cover here is the non-monotone spectral projected-gradient method \cite{birgin2000nonmonotone,davenport20121}.

\subsection{Generic methods}

Here we only focus on solving the matrix completion problem (\ref{optimization_problem}) by proximal-gradient method; the matrix recovery problem $(\ref{estimator})$ can be solved similarly as stated at the end of this subsection.

First, rewrite $\mathcal{S}$ in (\ref{searchspace}) as the intersection of two closed and convex sets in $\mathbb{R}^{d_1 \times d_2}$:
\begin{equation}
\begin{split}
\Gamma_1 \triangleq \{X \in \mathbb{R}^{d_1 \times d_2}: &\beta \leq X_{ij}\leq \alpha, \\
&\forall (i,j) \in \llbracket d_1\rrbracket\times\llbracket d_2\rrbracket\},\end{split}
\label{boxconstriant}
\end{equation}
and
\[
\Gamma_2 \triangleq \{X \in \mathbb{R}^{d_1 \times d_2} : \|X\|_* \leq \alpha\sqrt{rd_1 d_2}\},
\]
where the first set is a box and the second set is a nuclear norm ball. Let $f(X) \triangleq -F_{\Omega, Y}(X)$ be the negative log-likelihood function. Then optimization problem (\ref{optimization_problem}) is equivalent to
\begin{equation}
\widehat{M} = \arg \min_{X \in \Gamma_1 \bigcap \Gamma_2} f(X).
\label{newoptimizationproblem}
\end{equation}

Noticing that the search space $\mathcal{S} = \Gamma_1 \bigcap \Gamma_2$ is closed and convex and $f(X)$ is a convex function, we can use proximal gradient methods to solve (\ref{newoptimizationproblem}). Let $\mathbb{I}_{\Gamma}(X)$ be an extended function that takes value zero if $X \in \Gamma$ and value $\infty$ if $X \not\in \Gamma$. Then problem (\ref{newoptimizationproblem}) is equivalent to
\begin{equation}
\widehat{M} = \arg \min_{X \in \mathbb{R}^{d_1 \times d_2}} f(X) + \mathbb I_{\Gamma_1 \bigcap \Gamma_2}(X).
\end{equation}
To guarantee the convergence of proximal gradient method, we need the Lipschitz constant $L>0$, which satisfies
\begin{equation}
\| \nabla f(U) - \nabla f(V) \|_F \leq L \|U-V\|_F, \quad \forall U,V \in \mathcal{S},
\label{Lipschitz}
\end{equation}
and hence $L = \alpha/\beta^2$ by the definition of our problem. Define the orthogonal projection of a matrix $X$ onto a convex set $\widetilde{\Gamma}$ as
$$
\Pi_{\widetilde{\Gamma}}(X) \triangleq \arg \min_{Z \in \widetilde{\Gamma}} \| Z-X \|_F^2.
$$

\subsubsection{Proximal gradient}
Initialize the algorithm by $[X_0]_{ij} = Y_{ij}$ for $(i, j) \in \Omega$ and $[X_0]_{ij} = (\alpha+\beta)/2$ otherwise. Then iterate using
\begin{equation}
X_k = \Pi_{\mathcal{S}} (X_{k-1} - (1/L)\nabla f(X_{k-1})). \label{iter_proximal}
\end{equation}
This algorithm has a linear convergence rate:
\begin{lemma}[Convergence of proximal gradient]
Let $\{X_k\}$ be the sequence generated by (\ref{iter_proximal}). Then for any $k>1$, we have
$$
f(X_k) - f(\widehat{M}) \leq \frac{L \|X_0 - \widehat{M}\|_F^2}{2k}.
$$
\label{convergence1}
\end{lemma}

\subsubsection{Accelerated proximal gradient}
Although proximal gradient can be implemented easily, it converges slowly when the Lipschitz constant $L$ is large. In such scenarios, we may use Nesterov's accelerated method \cite{Ghaoui2010}. With the same initialization as above, we perform the following two projections at the $k$th iteration:
\begin{equation}
\begin{split}
X_k &= \Pi_{\mathcal{S}} (Z_{k-1} - (1/L)\nabla f(Z_{k-1})), \\
Z_k &= X_k + \left( (k-1)/(k+2)\right)(X_k-X_{k-1}).
\end{split}
\label{acce}
\end{equation}
Nesterov's accelerated method converges faster:
\begin{lemma}[Convergence of accelerated proximal gradient]
Let $\{ X_k \}$ be the sequence generated by (\ref{acce}). Then for any $k>1$, we have
$$
f(X_k) - f(\widehat{M}) \leq \frac{2L \|X_0 - \widehat{M}\|_F^2}{(k+1)^2}.
$$
\label{convergence2}
\end{lemma}

\subsubsection{Alternating projection}
To use the above two methods, we need to specify ways to perform projection
 onto the space $\mathcal{S}$. Since $\mathcal{S}$ is an intersection of two convex sets, we may use alternating projection to compute a sequence that converges to the intersection of $\Gamma_1$ and $\Gamma_2$. Let $U_0$ be the matrix to be projected onto $\mathcal{S}$. Specifically, the following two steps are performed at the $j$th iteration:
$
V_j = \Pi_{\Gamma_2} (U_{j-1}) ~ \mbox{and} ~ U_j = \Pi_{\Gamma_1} (V_{j}),
$
until $\|V_j-U_j\|_F$ is less than a user-specified error tolerance. Alternating projection is efficient if there exist some closed forms for projection onto the convex sets. Projection onto the box constraint $\Gamma_1$ is quite simple: $[\Pi_{\Gamma_1}(Y)]_{ij}$ assumes value $\beta$ if $Y_{ij} < \beta$ and assumes value $\alpha$ if $Y_{ij}>\alpha$, and otherwise maintains the same value $Y_{ij}$ if $\beta \leq Y_{ij} \leq \alpha$. Projection onto $\Gamma_2$, the nuclear norm ball, can be achieved by projecting the vector of singular values onto a $\ell_1$ norm ball via scaling \cite{cai2010singular} \cite{duchi2008efficient}.

\subsubsection{Algorithm for matrix recovery}
To solve (\ref{estimator}), similarly, we can assume $M$ is approximately low-rank with a bounded nuclear norm and each entry of $M$ is bounded by $\beta=c/m$ and $\alpha=I$ in (\ref{boxconstriant}) based on the assumptions of Theorem \ref{theoremforcompressive}. Similar steps can be applied hereafter if we consider an additional convex set
\begin{equation}
\Gamma_0 \triangleq \{M \in \mathbb{R}_{+}^{d_1 \times d_2} : \|M\|_{1,1}=I, [M]_{jk} > 0, \forall jk\}.
\label{gamma0}
\end{equation}

\subsection{Penalized maximum likelihood singular value threshold (PMLSVT)}

We also develop an algorithm, referred to as PMLSVT, which is tailored to solving our Poisson problems. \yx{PMLSVT differs from the classical projected gradient in that instead of computing the exact gradient, it approximates the cost function  by expanding it using a Taylor expansion up to the second order. The resulted approximate problem with a nuclear norm regularization term has a simple closed form solution using Theorem 2.1 in \cite{cai2010singular}. Therefore, PMLSVT does not perform gradient descent directly, but it has a simple form and good numerical accuracy as verified by numerical examples. }

The algorithm is similar to the fast iterative shrinkage-thresholding algorithm (FISTA) \cite{beck2009fast} and its
extension to matrix case with Frobenius error \cite{ji2009accelerated}.
Similar to the construction in \cite{rohde2011estimation} and \cite{wainwright2014structured}, using $\lambda_0$ and $\lambda_1$ as regularizing parameters and the convex sets $\Gamma_0$ and $\Gamma_1$ defined earlier in (\ref{gamma0}) and (\ref{boxconstriant}), we may rewrite (\ref{estimator}) and (\ref{optimization_problem}) as
\begin{equation}
\widehat{M} = \underset{X \in \Gamma_i} {\arg \min}~ f_i(X) + \lambda_i \|X\|_{*}, \quad i = 0,1,
\label{originaloptimizationproblem}
\end{equation}
respectively, where $f_0(X) = - \sum_{i=1}^m \left\{y_i \log [\cA X]_i - [\cA X]_i\right\}$ and $f_1(X) = -F_{\Omega,Y}(X)$.

The PMLSVT algorithm can be derived as follows (similar to
\cite{ji2009accelerated}). For simplicity, we denote $f(X)$ for the $f_0(X)$ or $f_1(X)$. In the $k$th iteration, we may form a Taylor expansion of $f(X)$ around $X_{k-1}$ while keeping up to second term and then solve 
\begin{equation}
    X_k = \underset{X}{\arg\min} \left[Q_{t_k}(X,X_{k-1}) + \lambda\|X\|_{*}\right],
\label{ouroptimizationproblem}
\end{equation}
with
\begin{align}
    Q_{t_k}(X,X_{k-1}) &\triangleq f(X_{k-1}) + \langle X-X_{k-1},\nabla f(X_{k-1}) \rangle \nonumber \\
    &~~~+ \frac{t_k}{2}\|X-X_{k-1}\|_F^{2}, \label{new}
\end{align}
where $\nabla f$ is the gradient of $f$, $t_k$  is the reciprocal of the step size at the $k$th iteration, which we will specify later.
By dropping and introducing terms independent of $M$ whenever needed (more details can be found in \cite{cao2014low}), (\ref{ouroptimizationproblem}) is equivalent to
\begin{equation}
\begin{split}
    &X_k = \nonumber\\
    &\underset{X}{\arg\min} \left[\frac{1}{2} \left\| X- \left( X_{k-1} - \frac{1}{t_k}\nabla f(X_{k-1}) \right) \right\|_{F}^{2} + \frac{\lambda}{t_k}\|X\|_{*}\right].
    \end{split}
\label{ourfinalproblem}
\end{equation}
Recall $d = \min\{d_1, d_2\}$. For a matrix $Z \in \mathbb{R}^{d_1\times d_2}$, let its singular value decomposition be $Z=U{\Sigma} V^\intercal$, where $U \in \mathbb{R}^{d_1\times d}$, $V \in \mathbb{R}^{d_2\times d}$,
${\Sigma}=\diag\{[\sigma_1,\ldots, \sigma_d]\transpose\}$, and $\sigma_i$ is a singular value of the matrix $Z$.
For each $\tau \geq0$, define the {singular value thresholding operator} as:
\begin{equation}
    D_{\tau}(Z) \triangleq U \mbox{diag}\{[(\sigma_1-\tau)^+, \ldots, (\sigma_d-\tau)^+]\transpose\}V^\intercal.
\end{equation}
The following lemma is proved in \cite{cai2010singular}:
\begin{lemma}[Theorem 2.1 in \cite{cai2010singular}]
For each $\tau\ \geq0$, and $Z\in\mathbb{R}^{d_1\times d_2}$:
\begin{equation}
    D_{\tau}(Z) = \underset{X \in \mathbb{R}^{d_1\times d_2}}{\arg \min} \left\{ \frac{1}{2}\|X-Z\|_{F}^2+\tau\|Z\|_{*} \right\}.
\end{equation}
\label{theorem_cai}
\end{lemma}
Due to Lemma \ref{theorem_cai}, the exact solution to (\ref{ourfinalproblem}) is given by
\begin{equation}
    X_k = D_{\lambda/t_k} \left( X_{k-1} - \frac{1}{t_k}\nabla f(X_{k-1}) \right).
\label{Watkthiteration}
\end{equation}

The PMLSVT algorithm is summarized in Algorithm \ref{alg:main}. The initialization of matrix recovery problem is suggested by \cite{jain2013low} and that of matrix completion problem is to choose an arbitrary element  in the set $\Gamma_1$.
For a matrix $Z$, define a projection of $Z$ onto $\Gamma_0$ as follows:
$$
\mathcal{P}(Z) \triangleq \frac{I}{\|(Z)^+\|_{1,1}}(Z)^+,
$$
where $(i,j)th$ entry of $(Z)^+$ is $(Z_{ij})^+$. 
In the algorithm description, $t$ is the reciprocal of the step size, $\eta > 1$ is a scale parameter to change the step size, and $K$ is the maximum number of iterations, which is  user specified: a larger $K$ leads to more accurate solution, and a small $K$ obtains the coarse solution quickly. 
If the cost function value does not decrease, the step size is shortened to change the singular values more conservatively.  The algorithm terminates when the absolute difference in the cost function values between two consecutive iterations is less than $0.5/K$. Convergence of the PMLSVT algorithm cannot be readily established; however, Lemma \ref{convergence1} and Lemma \ref{convergence2} above may shed some light on this.

\begin{algorithm}
  \caption{PMLSVT for Poisson matrix recovery and completion}
  \begin{algorithmic}[1]
    \STATE Initialize: The maximum number of iterations $K$, parameters $\alpha$, $\beta$, $\eta$, and $t$.

    $X \leftarrow \mathcal{P}(\sum_{i=1}^m y_i A_i)$ \COMMENT{matrix recovery}

    $[X]_{ij} \leftarrow Y_{ij}$ for $(i, j) \in \Omega$ and  $[X]_{ij} \leftarrow (\alpha+\beta)/2$ otherwise \COMMENT{matrix completion}
    \FOR{$k = 1, 2, \ldots K $}
    \STATE $C  \leftarrow X - (1/t)\nabla f(X)$
    \STATE $C = U\Sigma V^\intercal$ \COMMENT{singular value decomposition}
    \STATE $[\Sigma]_{ii}  \leftarrow ([\Sigma]_{ii}-\lambda/t)^{+}$, $i = 1, \ldots, d$
        \STATE $X' \leftarrow X$ \COMMENT{record previous step}
    \STATE $X \leftarrow \mathcal{P} \left(U \Sigma V^\intercal\right)$ \COMMENT{matrix recovery}

    $X\leftarrow \Pi_{\Gamma_1} \left(U \Sigma V^\intercal\right)$ \COMMENT{matrix completion}
    \STATE If $ f(X) > Q_t(X,X')$ then $t \leftarrow \eta t$, go to 4.
    \STATE If $|f(X) - Q_t(X,X')| < 0.5/K$ then exit;
    \ENDFOR
  \end{algorithmic}
  \label{alg:main}
\end{algorithm}



\begin{remark}\label{complexity}
At each iteration, the complexity of PMLSVT (Algorithm \ref{alg:main}) is on the order of $O(d_1^2 d_2 + d_2^3)$ (which comes from performing singular value decomposition). This is much lower than the complexity of solving an SDP,  which is on the order of $O(d_1^3+d_1d_2^3+d_1^2 d_2^2)$. In particular, for a $d$-by-$d$ matrix, PMLSVT algorithm has a complexity $\mathcal{O}(d^3)$, which is lower than the complexity $\mathcal{O}(d^4)$ of solving an SDP. One may also use an approximate SVD method\cite{lerman2012robust} and a better choice for step sizes \cite{beck2009fast} to accelerate PMLSVT.
\end{remark}

\section{Numerical examples}

We use our PMLSVT algorithm in all the examples below. 

\subsection{Synthetic data based on solar flare image}

\subsubsection{Matrix recovery}

In this section, we demonstrate the performance of PMLSVT on synthetic data based on a real solar flare image captured by the NASA SDO satellite (see \cite{XieHuang2013} for detailed explanations). \yx{We use this original solar flare image to form a ground truth matrix $M$, and then generate Poisson observations of $M$ as described in (\ref{obs}).} All the numerical examples are run on a laptop with 2.40Hz dual-core CPU and 8GB RAM.

The solar flare image is of size  $48$-by-$48$ and is shown in Fig. \ref{fig:solar}(a). To generate $M$, we break the image into 8-by-8 patches, vectorize the patches, and collect the resulted vectors into a 64-by-36 matrix $M_0$ ($d_1 = 64$ and $d_2 = 36$). Such a matrix can be well approximated by a low-rank matrix $\bar{M}$, as demonstrated in Fig. \ref{fig:solar}(b). It is a image formed by $\bar{M}$, when $\bar{M}$ is a rank 10 approximation of $M_0$. Note that visually the rank-10 approximation is very close to the original image. Below we use this rank-10 approximation $\bar{M}$ as the ground truth matrix. The intensity of the image in  is $I=\|\bar{M}\|=3.27\times 10^6$.

To vary SNR, we scale the image intensity by $\rho \geq 1$ (\yx{i.e., the matrix to be recovered $M = \rho \bar{M}$ and $\bar{M}$ is the original matrix}). \yx{Hence, SNR of Poisson observations for the scaled image is proportional to $\sqrt{\rho I}$. }

Below, we use the PMLSVT algorithm for recovery and the parameters are $t = 10^{-5}$, $\eta=1.1$, and $K=2500$. Fig. \ref{fig:solar}(c) and Fig. \ref{fig:solar}(d) contain the recovered image for a fixed number of measurements $m = 1000$ and $\lambda = 0.002$, when $\rho = 2$ and $\rho = 7$ respectively. Note that the higher the $\rho$ (and hence the higher the SNR) the less error in the recovered image. In Fig. \ref{fig:snr2}, $\rho$ increases from 1 to 9 and the normalized squared error decreases as $\rho$ increases (although the improvement is incremental after $\rho$ is greater than 3).

\begin{figure}[h]
\begin{center}
\begin{tabular} {cc}
\includegraphics[width = 0.3\linewidth]{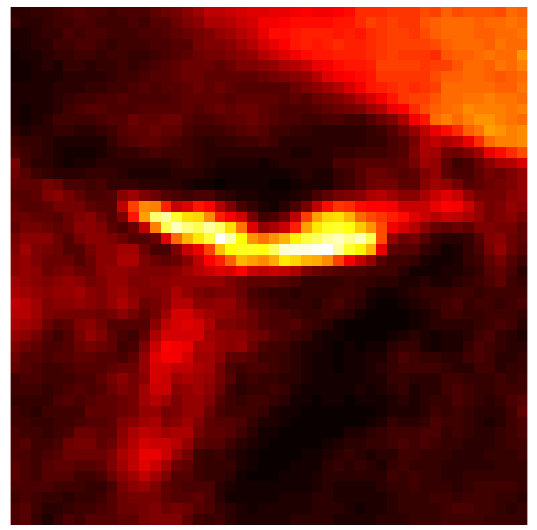}
& \includegraphics[width = 0.3\linewidth]{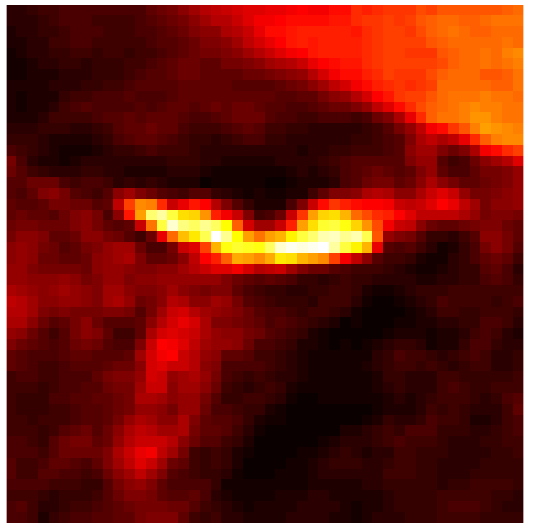} \\
(a) original.  & (b) rank 10 approximation. \\
\includegraphics[width = 0.3\linewidth]{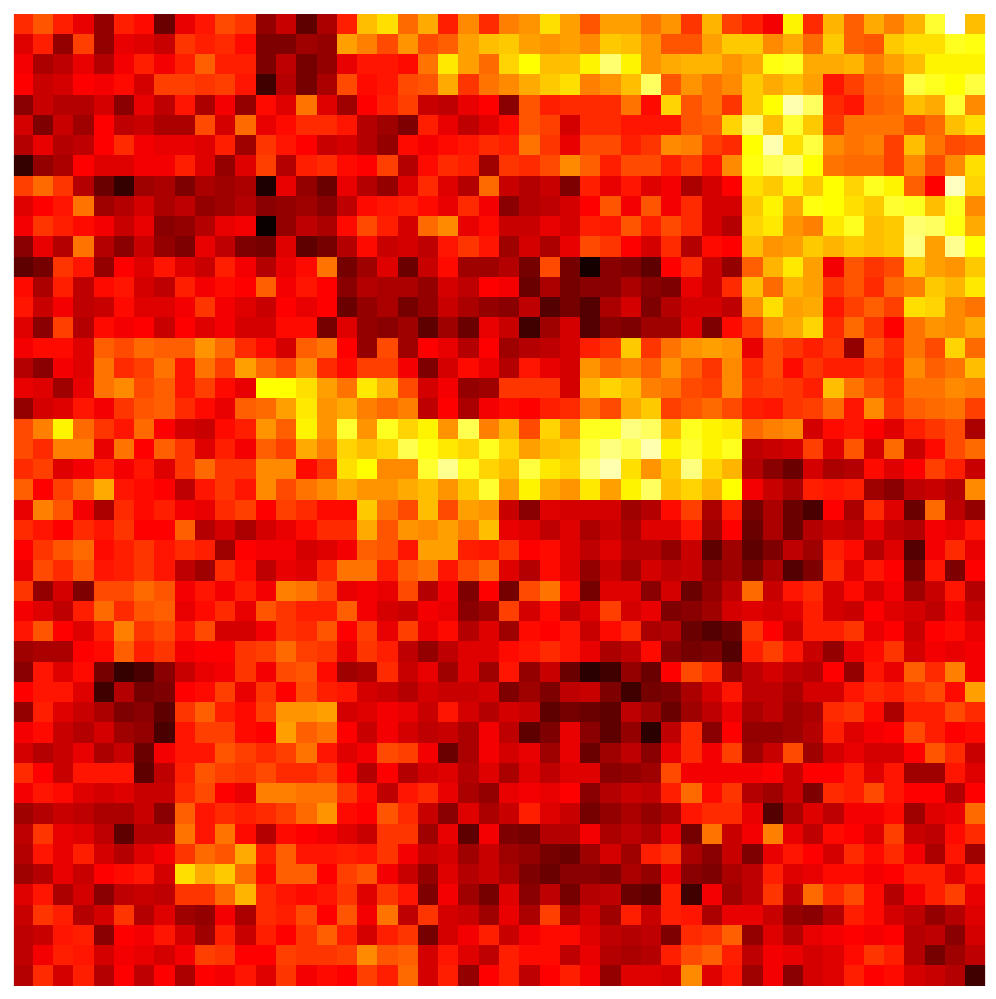}
& \includegraphics[width = 0.3\linewidth]{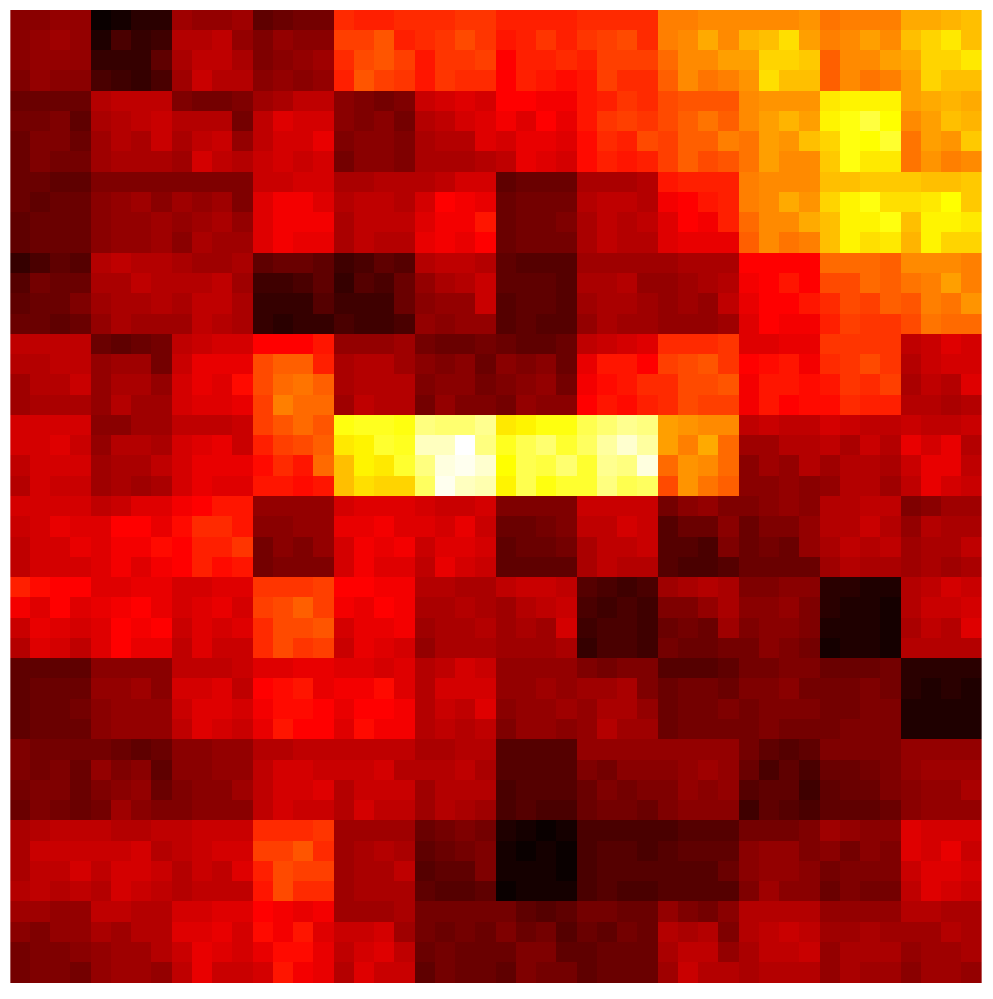} \\
(c) recovered, $\rho=2$. & (d) recovered, $\rho=7$.
\end{tabular}
\end{center}
\caption{Original, low-rank approximation to solar flare image, and recovered images from compressive measurements when the intensity of the underlying signal is scaled by $\rho$ (SNR is on the order or $1/\sqrt{\rho I}$). \yx{The parameters for (c) and (d) are $m = 1000$ and $\lambda = 0.002$. }}
\label{fig:solar}
\end{figure}

We further compare the quality of the recovered matrix using the PMLSVT algorithm (which approximately solves the maximum likelihood problem), with the recovered matrix obtained by solving the maximum likelihood problem {\it exactly} via semidefinite program (SDP) using CVX\footnote{http://cvxr.com/cvx/}. Solving via SDP requires a much higher complexity, as explained in Remark \ref{complexity}. Below, we fix $\rho=4$, $m = 1000$, $\lambda=0.002$, and increase the number of measurements $m$ while comparing the normalized square errors of these two approaches.  Fig. \ref{fig:riskvsmea} demonstrates that PMLSVT, though less accurate, has a performance very close to the exact solution via SDP. The increase in the normalized error of PMLSVT algorithm relative to that of SDP is at most 4.89\%. Also, the normalized errors of both approaches decrease as $m$ increases. PMLSVT is a lot faster than solving SDP by CVX, especially when $m$ is large, as shown in Table \ref{table1}. 

\begin{figure}[h]
\begin{center}
\includegraphics[width = 0.8\linewidth]{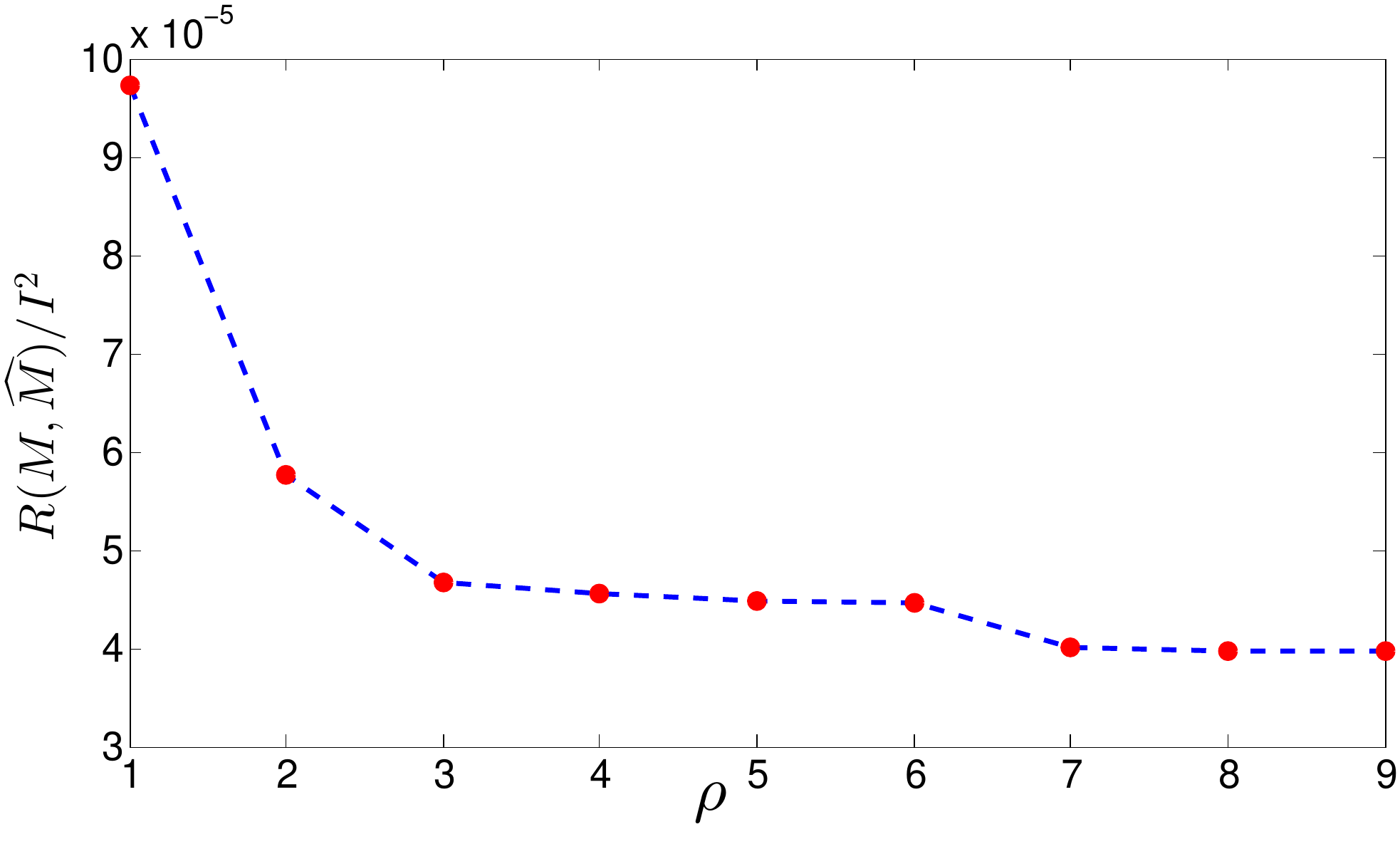}
\caption{Matrix recovery from compressive measurements: normalized error ${R}(M,\widehat{M})/I^2$ versus $\rho$, when $m=1000$ and $\lambda=0.002$, using PMLSVT. 
}
\label{fig:snr2}
\end{center}
\end{figure}

\begin{figure}[h]
\begin{center}
\includegraphics[width = 0.8\linewidth]{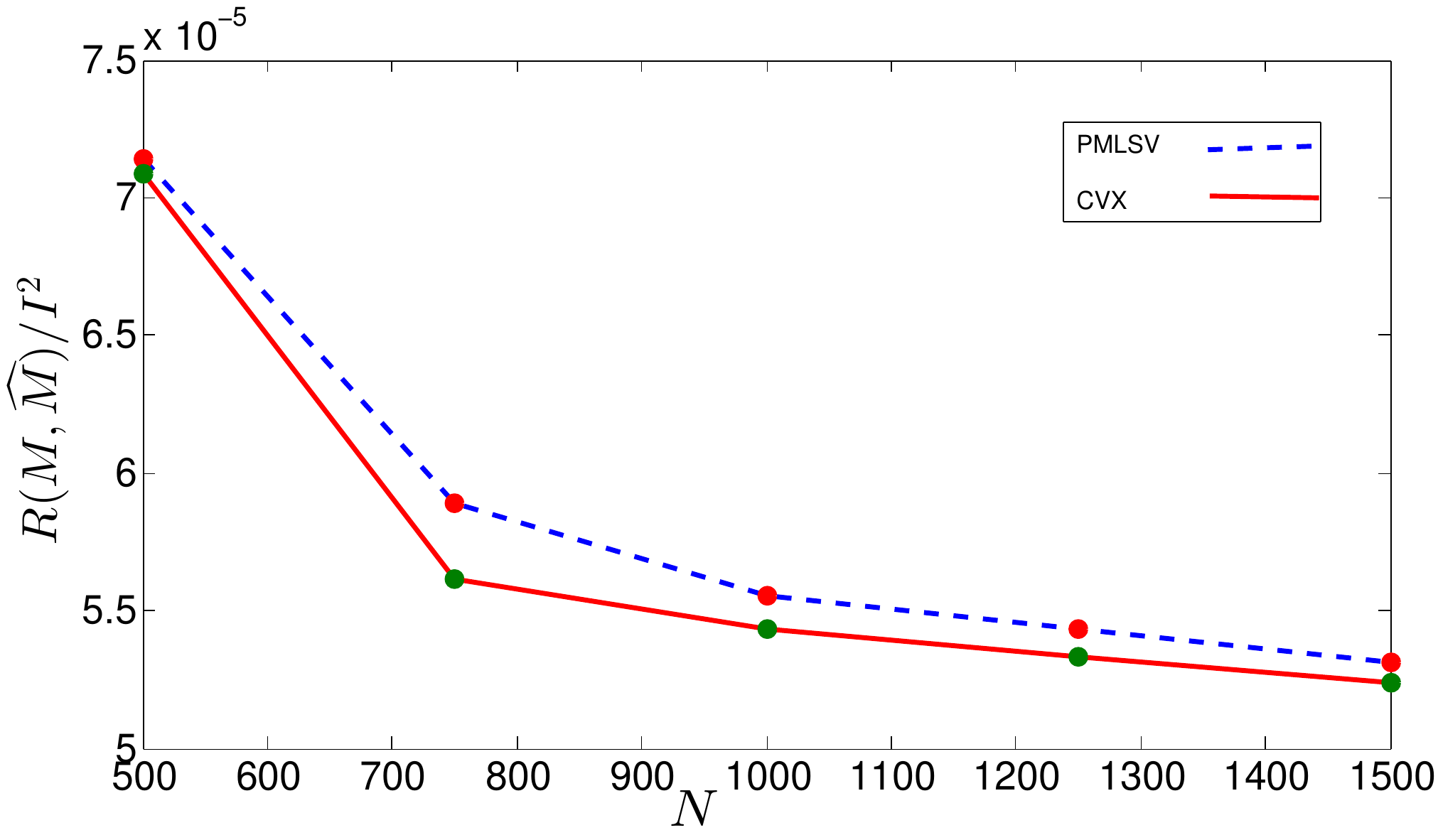}
\caption{Matrix recovery from compressive measurements: normalized error ${R}(M,\widehat{M})/I^2$ versus the number of measurements $m$ when $\rho=4$ and $\lambda=0.002$, for solutions obtained using CVX and PMLSVT, respectively.
}
\label{fig:riskvsmea}
\end{center}
\end{figure}

\begin{table}[h]
\center
\caption{CPU run time (in seconds) of solving SDP by using CVX and of the PMLSVT algorithms, with $\rho=4$ and $\lambda=0.002$ and $m$ measurements.}
\begin{tabular}{|l|c|c|c|c|c|}
  \hline
  $m$ & 500 & 750 & 1000 & 1250 & 1500 \\
  \hline
  SDP & 725 & 1146 & 1510 & 2059 & 2769 \\
  \hline
  PMLSVT & 172 & 232 & 378 & 490 & 642 \\
  \hline
\end{tabular}
\label{table1}
\end{table}

Fig. \ref{fig:lambda} demonstrates the normalized error with different values of $\lambda$, when $m=1000$ and $\rho=4$. Note that there is an optimal value for $\lambda$ with the smallest error (thus our choice for $\lambda = 0.002$ in the above examples).

\begin{figure}[h]
\begin{center}
\includegraphics[width = 0.8\linewidth]{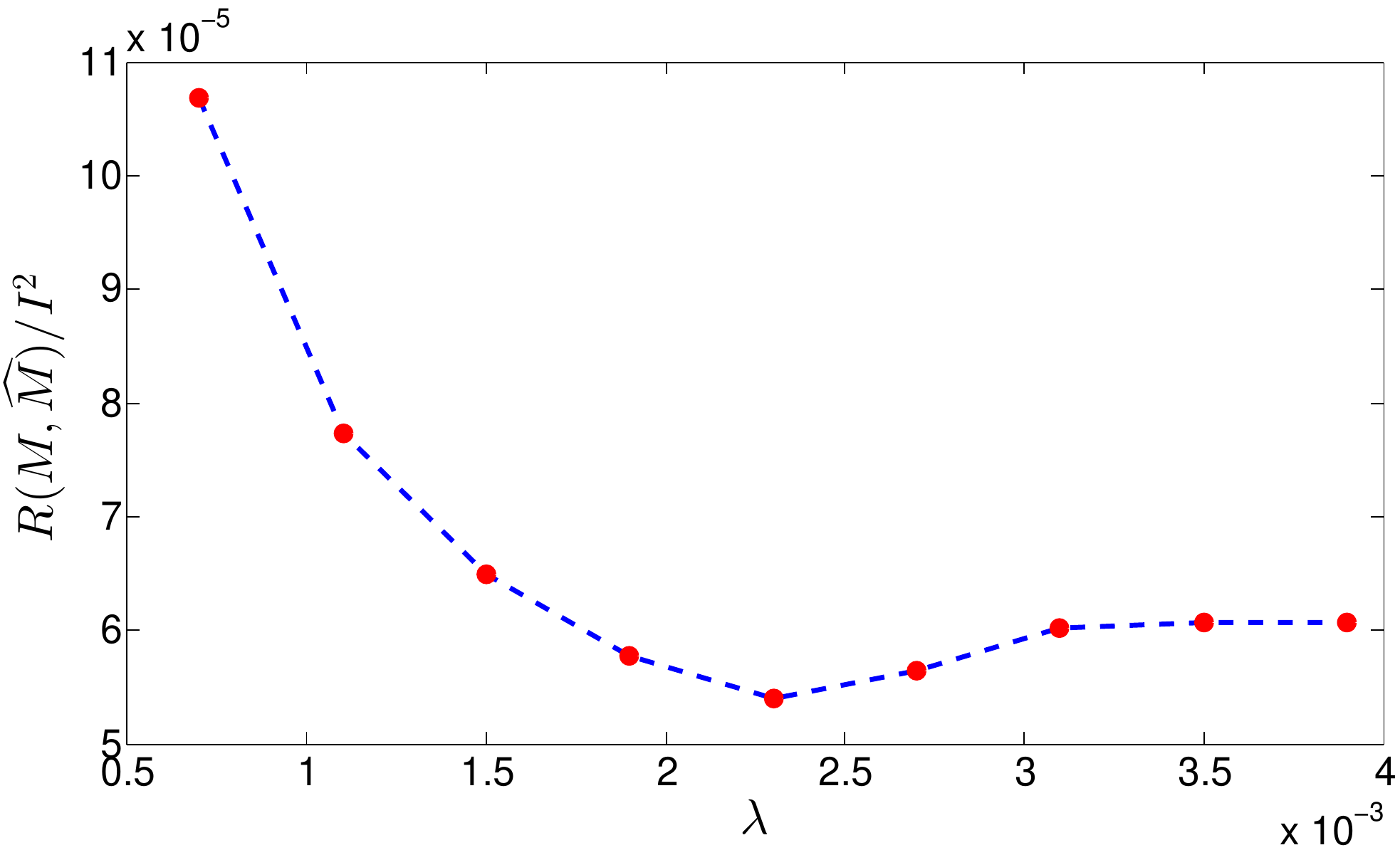}
\caption{Matrix recovery from compressive measurements: normalized error ${R}(M,\widehat{M})/I^2$ versus $\lambda$ when fixing $m=1000$ and $\rho=4$, using PMLSVT.
}
\label{fig:lambda}
\end{center}
\end{figure}

\vspace{.1in}
\subsubsection{Matrix completion}

We demonstrate the good performance of the PMLSVT algorithm for matrix completion on the same solar flare image as in the previous section. Set $\alpha = 200$ and $\beta = 1$ in this case. Suppose the entries are sampled via a Bernoulli model such that $\mathbb{E}[|\Omega|] = m$. Set $p \triangleq m/(d_1 d_2)$ in the sampling model. Set $t = 10^{-4}$ and $\eta=1.1$ for PMLSVT. Fig. \ref{fig:mc1} shows the results when roughly $80\%$, $50\%$ and $30\%$ of the matrix entries are observed. Even when about $50\%$ of the entries are missing, the recovery results is fairly good. When there are only about $30\%$ of the entries are observed, PMLSVT still may recover the main features in the image. PMLSVT is quite fast: the run times for all three examples are less than $1.2$ seconds.


\begin{figure}[h]
\begin{center}
\begin{tabular} {cc}
\includegraphics[width = 0.3\linewidth]{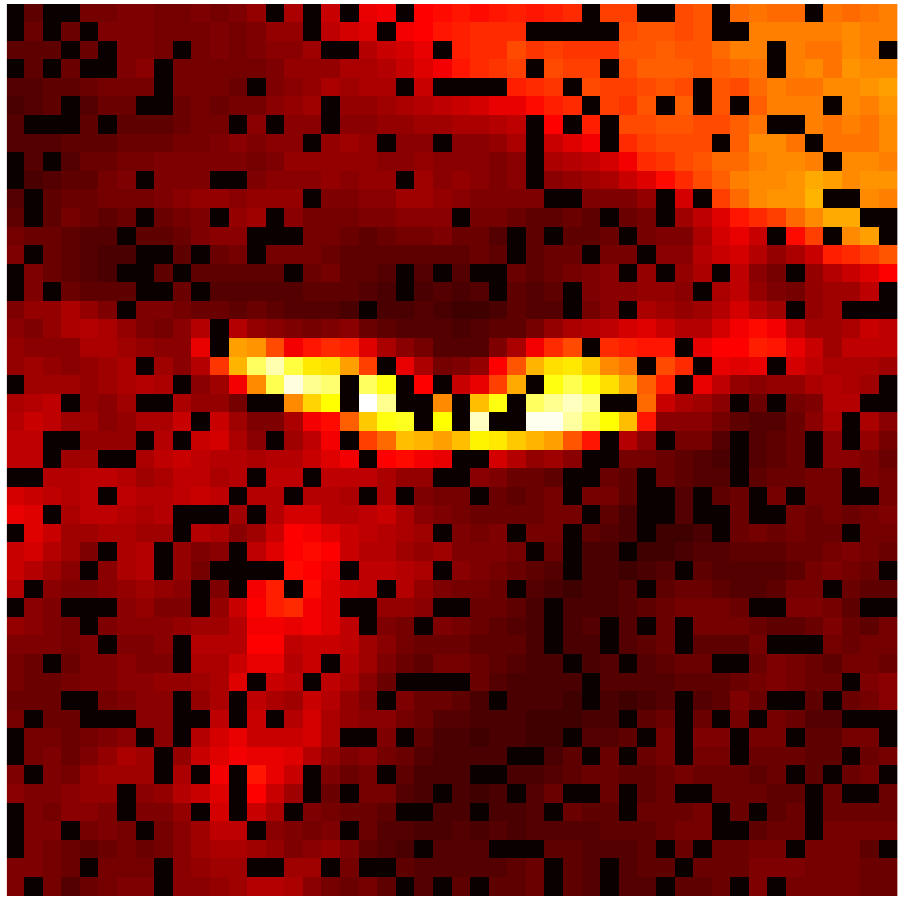} & \includegraphics[width = 0.3\linewidth]{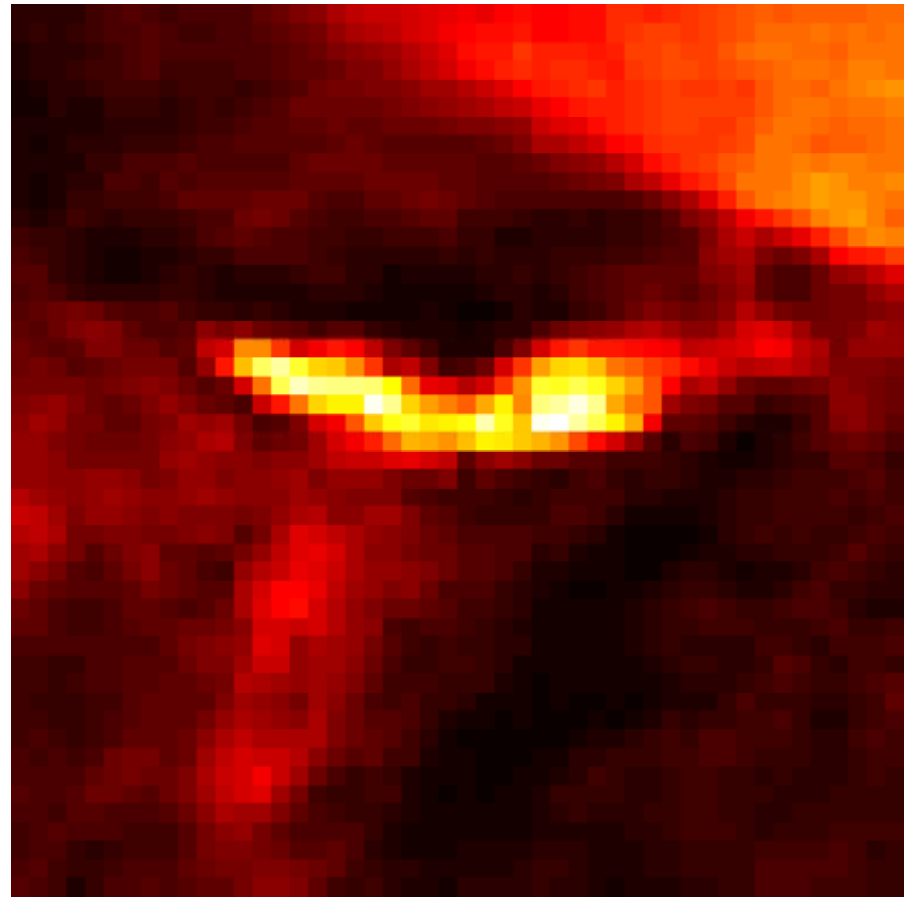} \\
(a) $p=0.8$. & (b) $\lambda=0.1, K=2000$. \\
\includegraphics[width = 0.3\linewidth]{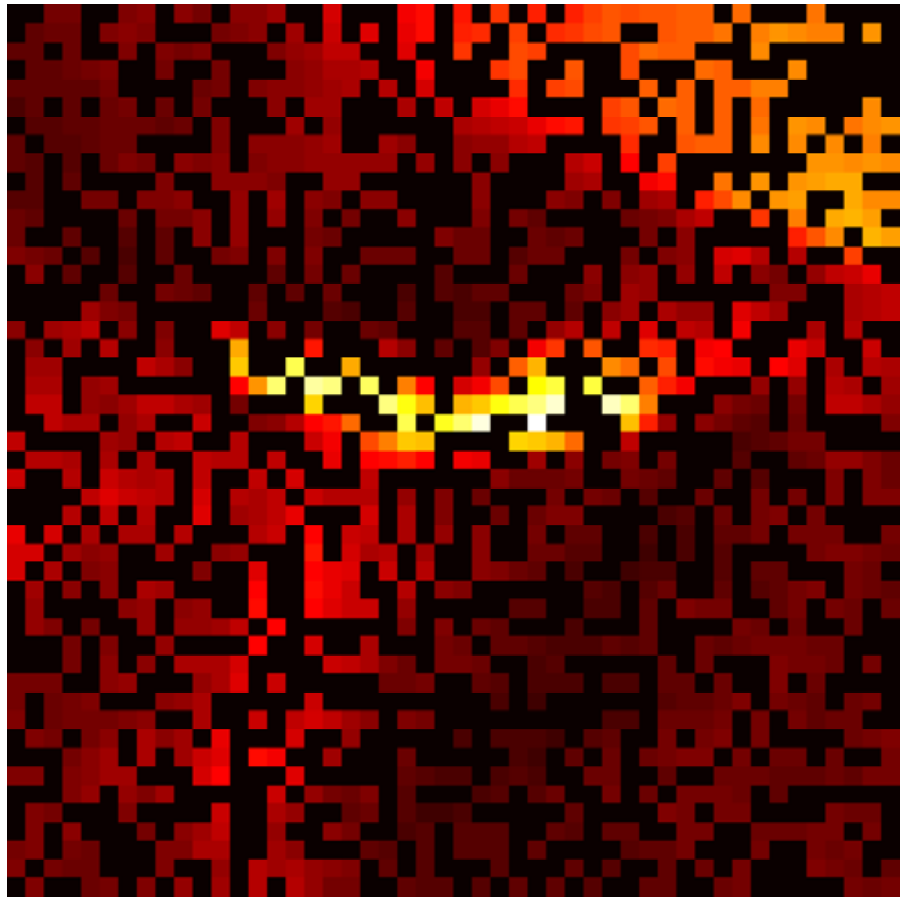} & \includegraphics[width = 0.3\linewidth]{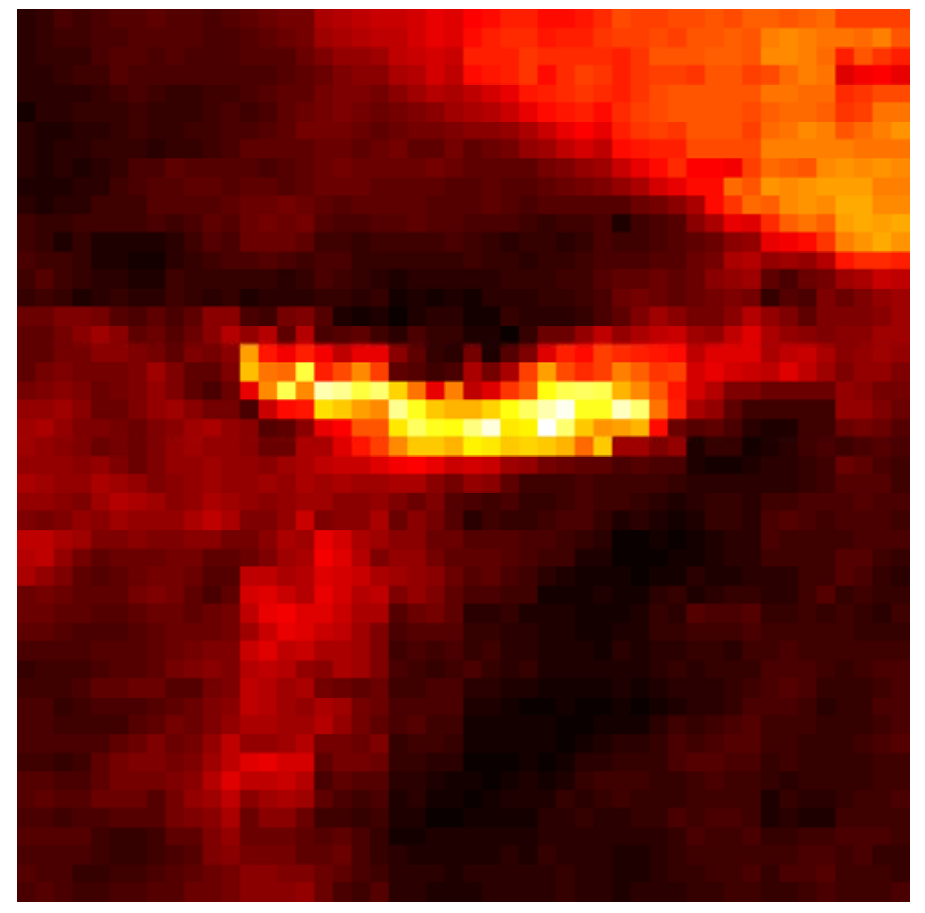} \\
(c) $p=0.5$. & (d) $\lambda=0.1, K=2000$. \\
\includegraphics[width = 0.3\linewidth]{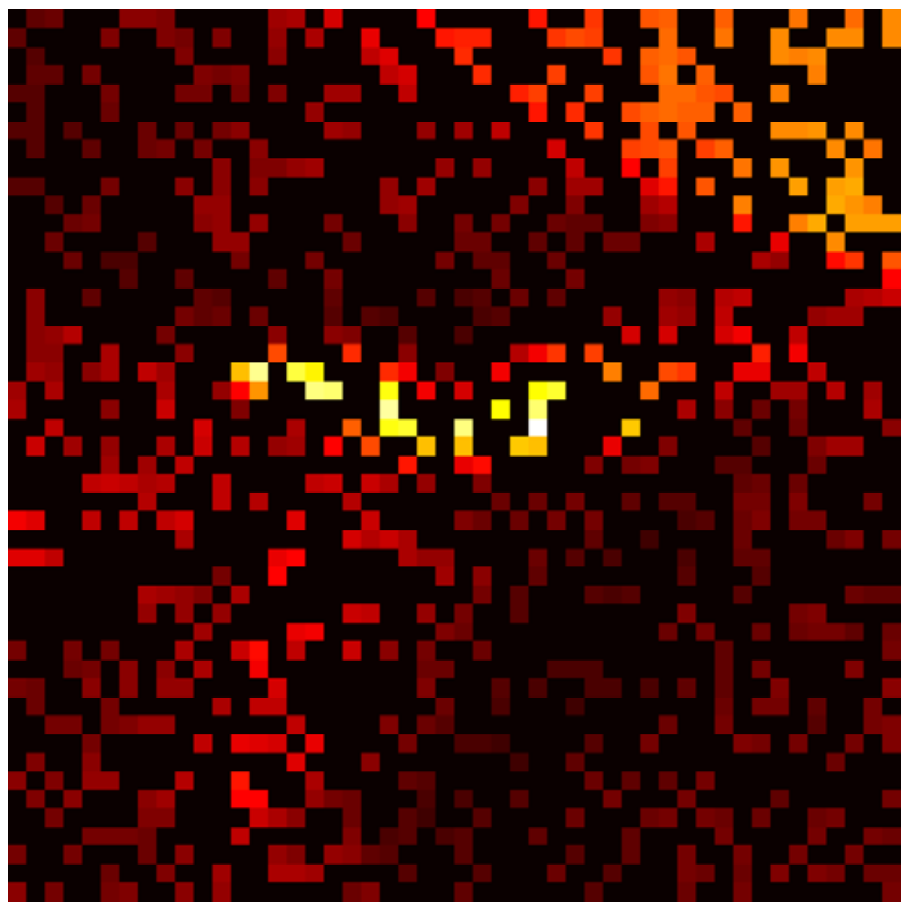} & \includegraphics[width = 0.3\linewidth]{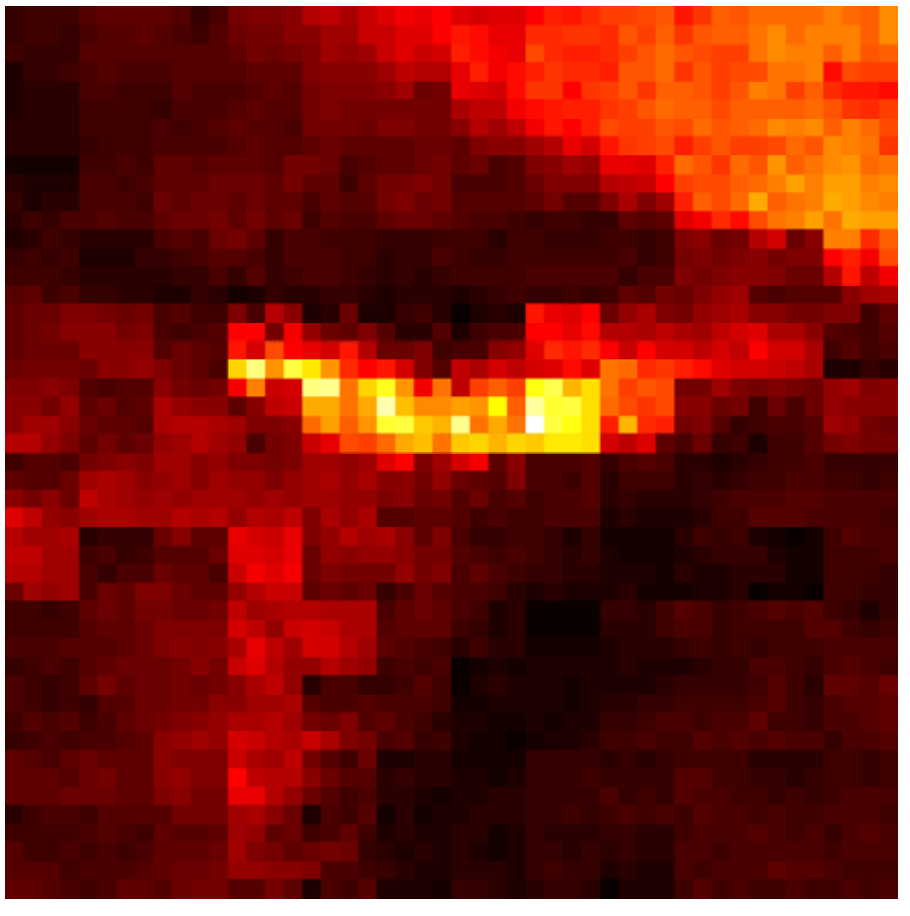} \\
(e) $p=0.3$. & (f) $\lambda=0.1, K=2000$.
\end{tabular}
\caption{Matrix completion from partial observations: (a), (c), and (e): $80\%$, $50\%$ and $30\%$ of entries observed (dark spots represent missing entries); (b), (d), and (f): images formed by complete matrix with $\lambda=0.1$ and no more than $2000$ iterations, and the run times of the PMLSVT algorithm are 1.176595, 1.110226 and 1.097281 seconds, respectively.}
\label{fig:mc1}
\end{center}
\end{figure}

\subsection{Bike sharing count data}
To demonstrate the performance of our algorithm on real data, we consider the bike sharing data set\footnote{The data can be downloaded at\\ http://archive.ics.uci.edu/ml/datasets/Bike+Sharing+Dataset\cite{Fanaee2013}.}, which  consists of $17379$ bike sharing counts aggregated on hourly basis between the years 2011 and 2012 in Capital bike share system with the corresponding weather and seasonal information. We collect countings of $24$ hours over $105$ Saturdays into a $24$-by-$105$ matrix $M$ ($d_1 = 24$ and $d_2 = 105$). The resulted matrix is nearly low-rank.
Assuming that only a fraction of the entries of this matrix are known (each entry is observed with probability 0.5 and, hence, roughly half of the entries are observed), and that the counting numbers follow Poisson distributions with unknown intensities.  We aim at recover the unknown intensities, i.e., filling the missing data and performing denoising.
We use PMLSVT with the following parameters: $\alpha=1000$, $\beta=1$, $t=10^{-4}$, $\eta=1.1$, $K=4000$ and $\lambda=100$.
In this case there is no ``ground truth'' for the intensities, and it is hard to measure the accuracy of recovered matrix. Instead, we are interested in identifying interesting patterns in the recovered results. As shown in Fig. \ref{fig:bike}(b), there are two clear increases in the counting numbers after the $17$th and the $63$th Saturday, which may not be easily identified from the original data in Fig. \ref{fig:bike}(a) with missing data and Poisson randomness.

\begin{figure}[h]
\begin{center}
\begin{tabular} {cc}
\includegraphics[width = 0.45\linewidth]{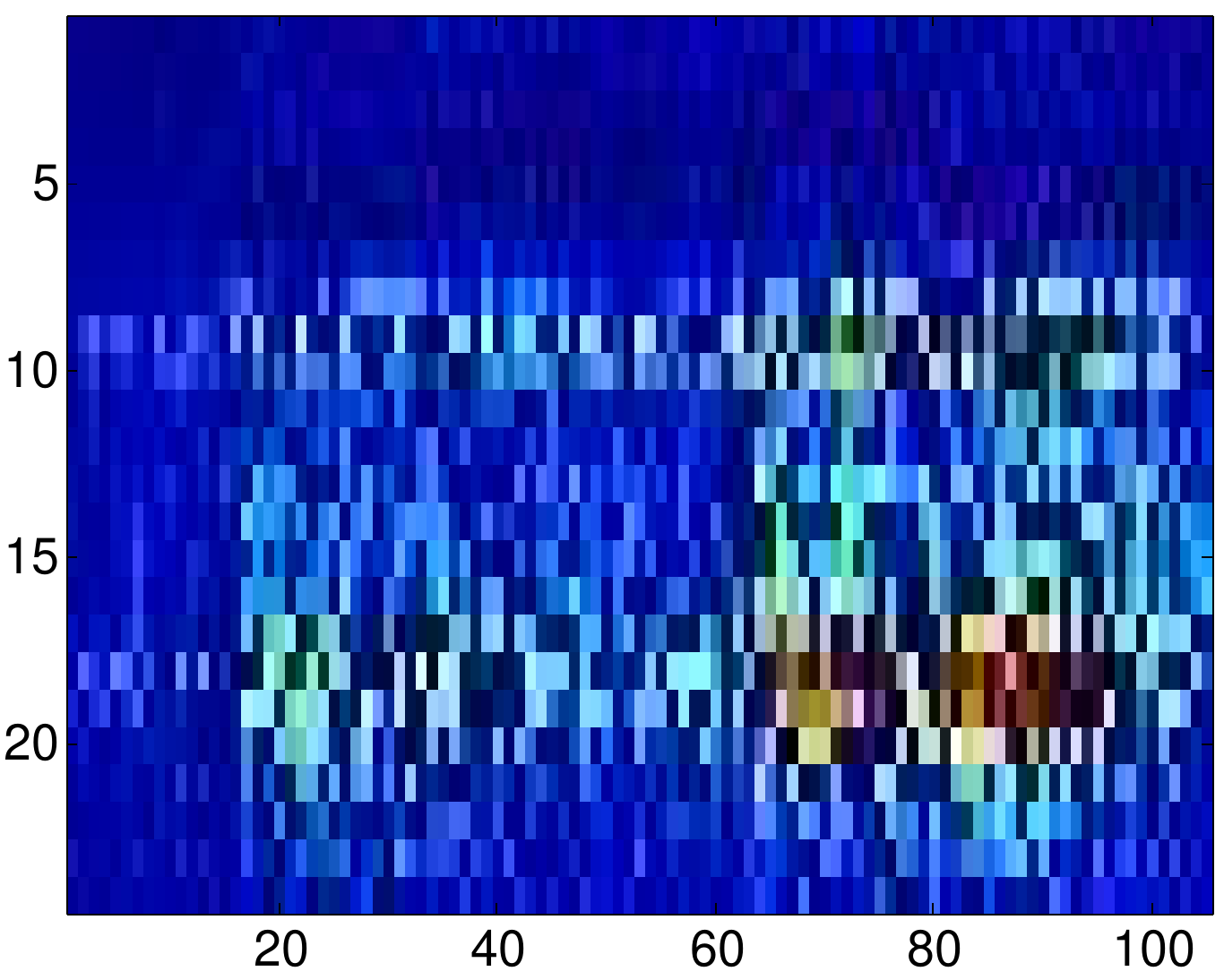} & \includegraphics[width = 0.45\linewidth]{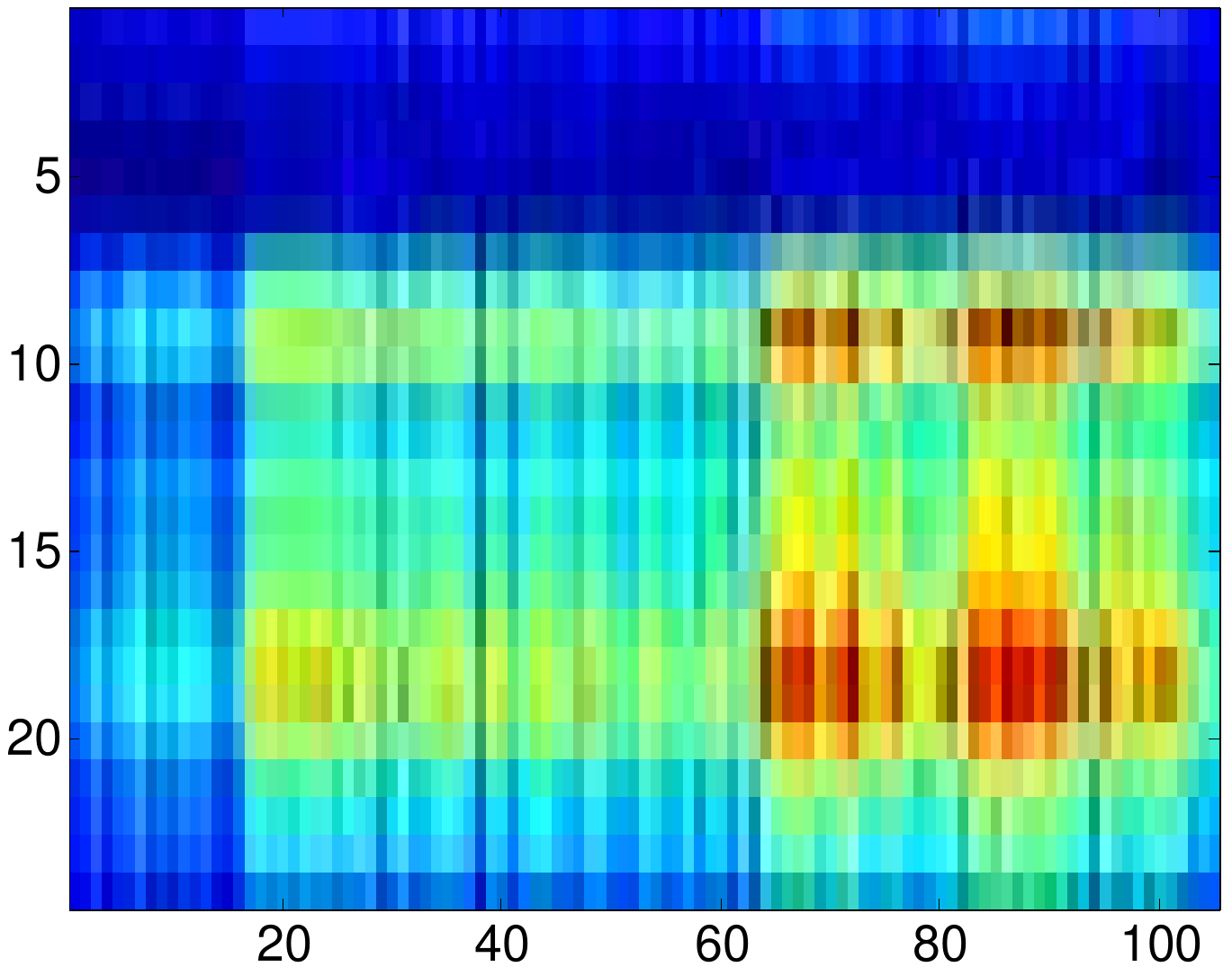} \\
(a) original data, $p=0.5$. & (b) $\lambda=100, K=4000$. \\
\end{tabular}
\caption{Bike sharing count data: (a): observed matrix $M$ with $50\%$ missing entries; (b): recovered matrix with $\lambda=100$ and $4000$ iterations, with an elapsed time of 3.147153 seconds.}
\label{fig:bike}
\end{center}
\end{figure}

\section{Conclusions}\label{sec:conclusion}
In this paper, we have studied matrix recovery and completion problem when the data are Poisson random counts. We considered a maximum likelihood formulation with constrained nuclear norm of the matrix and entries of the matrix, and presented upper and lower bounds for the proposed estimators. We also developed a set of new algorithms, and in particular the efficient the Poisson noise Maximal Likelihood Singular Value Thresholding (PMLSV) algorithm. We have demonstrated its accuracy and efficiency compared with the semi-definite program (SDP) and tested on real data examples of solar flare images and bike sharing data.

\section*{acknowledgement}

The authors would like to thank Prof. Yuejie Chi, Prof. Robert Calderbank, Prof. Mark Davenport, and Prof. Yaniv Plan for stimulating discussions and inspiring comments.

\bibliography{Poisson_MC}

\clearpage

\appendices

\section{Proofs for matrix recovery}

\begin{proof}[Proof of Lemma \ref{thm_RIP}]
Given matrices $X_1$ and $X_2$, we can define $u_i = \mbox{vec}(X_i)$, for $i = 1, 2$. Define the matrix
\[\tilde{A} = \begin{bmatrix}
\mbox{vec}(\tilde{A}_1)\transpose \\\vdots \\
\mbox{vec}(\tilde{A}_m)\transpose
\end{bmatrix} \in \mathbb{R}^{m\times (d_1 d_2)},
\]
then we have $\tilde{\mathcal{A}} X_i  = \tilde{A} u_i$ for $i = 1, 2$. Note that the matrix $\tilde{A}$ can be written as $A= Z/\sqrt{m}$, where the entries of  $Z$ are drawn i.i.d. to take value $-\sqrt{(1-p)/p}$ with probability $p$ or to take value $\sqrt{p/(1-p)}$ with probability $1-p$.
Also, by noticing the correspondence between vector and matrix norms, we have $\|X_1 - X_2\|_F  = \|u_1 - u_2\|_2$, $\|\tilde{\mathcal{A}} X_1 - \tilde{\mathcal{A}}X_2 \|_2 = \|\tilde{A} u_1 - \tilde{A} u_2\|_2$. Moreover, since $\|X_i\|_{1, 1} = \|u_i\|_1$ and $\|X_i\|_F = \|u_i\|_2$, for $i = 1, 2$, the $\ell_1$-ball $\mathcal{B}^{d_1\times d_2}$ and the sphere $\mathcal{S}^{d_1\times d_2 - 1}$ defined for matrix of dimension $d_1$-by-$d_2$ can be translated to $\ell_1$-ball and sphere for the corresponding vector space of dimension $d_1 d_2$. Hence, all conditions of Theorem 1 \cite{raginsky2010compressed} are satisfied, and we may apply it to a signal vector space of dimension $d_1 d_2$ and a matrix operator of dimension $m$-by-$(d_1 d_2)$ to obtain the statement in
Lemma \ref{thm_RIP}.

\end{proof}

\begin{proof}[Proof of Lemma \ref{thm_general}]
Assume $M$ and $\widehat{M}$ are the true matrix and its estimator, respectively. Let $u \triangleq \mbox{vec}(M)$ and $\hat{u} \triangleq \mbox{vec}(\widehat{M})$. Again, by making the links that $\tilde{\mathcal{A}}M = \tilde{A} u$, $\tilde{\mathcal{A}}\widehat{M} = \tilde{A} \hat{u}$, $\|\tilde{\mathcal{A}}(M - \widehat{M})\|_2 = \|\tilde{A}(u - \hat{u})\|_2$, $\|M\|_{1, 1} = \|\widehat{M}\|_{1, 1} = I$ is equivalently $\|u\|_1 = \|\widehat{u}\|_1 = I$, as well as the RIP for the matrix operator Lemma \ref{thm_RIP}, we may directly apply Theorem 2 in \cite{raginsky2010compressed} to the case of a vector of dimension $d_1 d_2$ and a matrix operator of dimension $m$-by-$d_1 d_2$ to obtain the desired result.
\end{proof}

The proof of Theorem \ref{regret_for_nearly_low_rank} requires the following lemma:
\begin{lemma}[Covering number for low-rank matrices, Lemma 4.3.1. in \cite{PlanThesis2011}] \label{yaniv}
Let $S_r = \{X\in \mathbb{R}^{d_1\times d_2}: \mbox{rank}(X)\leq r, \|X\|_F = 1\}$, then there exists an $\varepsilon$-net $\bar{S}_r \subset S_r$, with respect to Frobenius norm, i.e., for any $V \in S_r$, there exists $V_0 \in \bar{S}_r$, such that $\|V_0 - V\|_F\leq \varepsilon$, and
\[
|\bar{S}_r| \leq \left(\frac{9}{\varepsilon}\right)^{(d_1 + d_2 +1)r}
\]
\end{lemma}

\begin{proof}[Proof of Theorem \ref{regret_for_nearly_low_rank}]
The proof of Theorem \ref{regret_for_nearly_low_rank} involves constructing a suitable set of estimators $\Gamma_0$ and $\mathcal{G}$ for Lemma \ref{thm_general}, estimate the sizes of the set, set the regularization function $\pen(X)$ such that it encourages low-rank $M$ and satisfies Kraft inequality, and then invoking Lemma \ref{thm_general}.
Given $X \in \Gamma$, we introduce its scaled version
 \[\bar{X}^{(\ell)} = U \diag\{\theta^{(\ell)}/\|\theta^{(\ell)}\|_2\}{V}\transpose,\]
so that $\|\bar{X}^{(\ell)}\|_F = 1$ and
\ben
X^{(\ell)} = \|\theta^{(\ell)}\|_2 \bar{X}^{(\ell)}.  \label{quant}
\een
Since all $X \in \Gamma$ satisfies $\|X\|_{1, 1} = I$ and $X_{jk} \geq 0$, we have that
\begin{equation}
\begin{split}
&\|\theta^{(\ell)}\|_2 \leq \|\theta\|_2 =\sqrt{\sum_{i=1}^d \theta_i^2} = \sqrt{\tr(X\transpose X)}\\
&= \sqrt{\sum_j \sum_k X_{jk}^2} \leq \sum_j \sum_k X_{jk} = I.
\end{split}
\label{theta_I}
\end{equation}
Using the parameterization in (\ref{quant}), we can code $X^{(\ell)}$ using  three steps by encoding the ``magnitude'' $\|\theta^{(\ell)}\|_2$, the scaled rank-$\ell$ matrices $\bar{X}^\ell$, and the value $\ell$ of the rank itself.
(1) Quantize $\|\theta^{(\ell)}\|$ into one of $\sqrt{d}$ bins that uniformly divide the interval $[-I, I]$. Let the result of the quantization to be $r_q$. Since there are $\sqrt{d}$ bins, encoding $r_q$ requires $\frac{1}{2} \log_2 d$ bits.
(2) Quantize $\bar{X}^{(\ell)}$. Since $\|\bar{X}^{(\ell)}\|_F = 1$, using Lemma \ref{yaniv}, we can form a $\varepsilon$-net $S_q$ such that for every $\bar{M}^{(\ell)}$, there is a corresponding $\bar{X}^{(\ell)}_q \in S_q$ with $\|\bar{X}^{(\ell)}-\bar{X}^{(\ell)}_q\|_F \leq 9/\sqrt{d}$ and $|S_q| = d^{(d_1 + d_2 + 1)\ell/2}$. Hence, to encode the elements in $S_q$, we need $\frac{1}{2}(d_1+ d_2 + 1)\ell \log_2 d$ bits.
(3) Finally, encode $\ell$, the rank of $\bar{X}^{(\ell)}$. Since the rank of these matrices are at most, we need $\log_2 d$ bits.
Let \begin{equation}
X^{(\ell)}_q \triangleq r_q \bar{X}^{(\ell)}_q.\label{bar_M_q}\end{equation}
It can be verified that the above quantization scheme results in a set of approximations for $X^{(\ell)}$ and a corresponding prefix code for $\bar{X}^{(\ell)}$. From the three steps above, the average code length for $\bar{X}^{(\ell)}$ is upper bounded by $\frac{3}{2}\log_2 d + \frac{1}{2}(d_1+d_2 +1)\ell\log_2 d$ bits.
Finally, to ensure total intensity constraint and that each element of $X$ is greater than $c$, we project ${X}^{(\ell)}_q$ onto a set
\[
\mathcal{C} \triangleq \{X\in \mathbb{R}_+^{d_1\times d_2}: X_{jk} \geq c \mbox{ and } \|X\|_{1, 1} = I.\}
\]
and use $\mathcal{P}_{\mathcal{C}}({X}^{(\ell)}_q)$ as a candidate estimator in $\Gamma$, where $\mathcal{P}_{\mathcal{C}}$ is a projector operator onto the set $\mathcal{C}$, i.e.,
\yx{
\[
\mathcal{P}_{\mathcal{C}}(X) = \underset{X' \in \mathcal{C}}{\arg\min} \|X' - X\|_F.
\]
}
 Using the construction above for $\Gamma$, the complexity of $X \in \Gamma$ satisfies
\[
\pen({X}) \leq \frac{3}{2}\log_2 d + \frac{1}{2}(d_1+d_2 +1)\ell\log_2 d < \frac{1}{2}(d_1 + d_2 + 4)\ell \log_2 d.
\]

Now, given $X  = U^* \diag\left\{\theta^*\right\} {V^*}\transpose$, and let $X^{(\ell)}$ be its best rank-$\ell$ approximation, $X_q^{(\ell)}$ be the quantized version of $X^{(\ell)}$, for which we have
\begin{align}
&~ \|X^{(\ell)}-X^{(\ell)}_q\|_F^2 \\
 = &
~\| \|\theta^{(\ell)}\|_2 \bar{X}^{(\ell)}
- r_q\bar{X}^{(\ell)}+ r_q\bar{X}^{(\ell)}
-r_q \bar{X}^{(\ell)}_q\|_F^2
\\
\leq & ~ 2 \| \|\theta^{(\ell)}\|_2 \bar{X}^{(\ell)}
- r_q\bar{X}^{(\ell)}\|_F^2
+ 2\| r_q\bar{X}^{(\ell)}
-r_q \bar{X}^{(\ell)}_q\|_F^2 \\
=  & ~2(\|\theta^{(\ell)}\|_2 - r_q)^2 \| \bar{X}^{(\ell)}\|_F^2
+ 2 r_q^2 \|\bar{X}^{(\ell)}
- \bar{X}^{(\ell)}_q\|_F^2 \\
\leq  &~ 2\frac{I^2}{d} + 2I^2 \frac{81}{d} = \frac{164 I^2}{d},
\label{chain_M}
\end{align}
where the last inequality follows from $|\|\theta^{(\ell)}\|_2 - r_q|<2I/(2\sqrt{d})$, $\|\bar{X}^{(\ell)}\|_F = 1$, $|r_q|\leq I$, and $\|\bar{X}^{(\ell)}-\bar{X}^{(\ell)}_q\|_F \leq 9/\sqrt{d}$.
Hence, using above, we can bound the distance between an arbitrary matrix $X$ and its candidate estimator $\mathcal{P}_{\mathcal{C}}({X}^{(\ell)}_q) \in \Gamma$, 
\yx{
\begin{align*}
& ~\|X - \mathcal{P}_{\mathcal{C}}({X}^{(\ell)}_q) \|_F^2 \\
=&~\|X - {X}^{(\ell)}_q + {X}^{(l)}_q - \mathcal{P}_{\mathcal{C}}({X}^{(\ell)}_q) \|_F^2 \\
\leq&~ 2  \|X - {X}^{(\ell)}_q \|_F^2 + 2\|{X}^{(\ell)}_q - \mathcal{P}_{\mathcal{C}}({X}^{(\ell)}_q) \|_F^2 \\
\leq &~ 2 \| X - {X}^{(\ell)} + {X}^{(l)} - {X}^{(\ell)}_q \|_F^2 + 8I^2\\
\leq&~  4  \| X - {X}^{(\ell)} \|_F^2 + 4 \|{X}^{(\ell)} - {X}^{(\ell)}_q \|_F^2 + 8I^2 \\
\leq&~ 4 I^2 (c_0 \varrho^2 \ell^{-2(1/q - 1/2)} + 164/d + 2),
\end{align*}
for some constant $c_0 > 0$ that depends only on $q$, where the second step is due to triangle inequality,
\begin{align*}
\|{X}^{(\ell)}_q - \mathcal{P}_{\mathcal{C}}({X}^{(\ell)}_q) \|_F^2
\leq & ~2 \|{X}^{(\ell)}_q\|_F^2 + 2 \|\mathcal{P}_{\mathcal{C}}({X}^{(\ell)}_q)\|_F^2 \\
\leq &~ 2 r_q^2 + 2I^2 \leq 4I^2.
\end{align*}
Since $\mathcal{P}_{\mathcal{C}}({X}^{(\ell)}_q) \in \mathcal{C}$, so $\|\mathcal{P}_{\mathcal{C}}({X}^{(\ell)}_q) \|_{1, 1} = I$. Then due to an argument similar to the last inequality in (\ref{theta_I}),
\[ \|\mathcal{P}_{\mathcal{C}}({X}^{(\ell)}_q)\|_F \leq  \|\mathcal{P}_{\mathcal{C}}({X}^{(\ell)}_q)\|_{1, 1} = I.\] Above, we have used the inequality $\|X + Y\|_F^2 \leq 2\|X\|_F^2  + 2\|Y\|_F^2$. The last inequality uses (\ref{chain_M}) and our geometrically decay singular value assumption and (\ref{approx_inq}).}

Given each $1\leq \ell \leq d$, let $\hat{\Gamma}_\ell$ be a set such that the the elements in $\hat{\Gamma}_\ell$ are rank-$\ell$ and after projection onto set $\mathcal{C}$ they belong to $\Gamma$. Then from the estimate above $\log |\hat{\Gamma}_\ell| < \frac{1}{2} (d_1 + d_2 + 4) \ell \log _2 d$. Therefore, $\hat{\Gamma}_\ell \in \mathcal{G}$, whenever
\[\frac{1}{2} (d_1 + d_2 + 4) \ell \log_2 d\leq m/(c_4 \xi_p^4),\]
for
\[
1\leq \ell\leq \ell_*, \quad\mbox{where}\quad \ell_* \triangleq \frac{2m}{c_4 \xi_p^4(d_1 + d_2 + 4) \log_2 d},
\]
and we may invoke Lemma \ref{thm_general}.
The term $164/d$ is independent of $\ell$, so we bring it out from the maximization with respect to $\ell$.
Finally, under this condition, using the statement of Lemma \ref{thm_general}, for a nearly low-rank matrix $M$, suppose its estimator constructed as above is given by $\mathcal{P}_{\mathcal{C}}(M_q^{(\ell)})$, we have
\begin{align*}
\frac{1}{I^2}\mathbb{E}[R(M, \widehat{M})  ]&
\leq C_{m, p}\min_{1\leq \ell\leq \ell_*} (\frac{1}{I^2}\|M - \mathcal{P}_{\mathcal{C}}(M_q^{(\ell)})\|_F^2 \\
& ~~~+ \frac{\lambda \pen(M_q^{(\ell)})}{I})+ \frac{2c_2^2 \xi_p^4 \log (c_2 \xi_p^4 d_1 d_2/m)}{m} \\
& \leq \mathcal{O}(m) ( 164/d + 2 + \min_{1\leq \ell\leq \ell_*} [c_0 \varrho^2 \ell^{-2(1/q - 1/2)}  \\
& ~~~ + \frac{\lambda (d_1+d_2+4) \ell\log_2 d}{2I}])\\
&\quad + \mathcal{O}\left(\frac{\log(d_1d_2/m)}{m}\right),
\end{align*}
where the second inequality is because $C_{m,p}$ is $\mathcal{O}(m)$.
\end{proof}

\section{Proofs for matrix completion}
\yc{
To begin, we first recall some definitions from introduction and explain some
additional notation that we will need for the proofs. For two probability distributions $\mathcal{P}$ and $\mathcal{Q}$ on a
countable set $A$, $D(\mathcal{P}||\mathcal{Q})$ will denote the Kullback-Leibler (KL) divergence
\[
D(\mathcal{P}||\mathcal{Q}) = \sum_{x\in A}\mathcal{P}(x) \log\left( \frac{\mathcal{P}(x)}{\mathcal{Q}(x)} \right),
\]
where $\mathcal{P}(x)$ denotes the probability of the outcome $x$ under the distribution $\mathcal{P}$. In the following, we will abuse this notation slightly, to mean the KL divergence between two Poisson distributions with different parameters (the arguments in the notations denote parameters of the Poisson distributions), in the following two ways. First, for scalar inputs $p,q \in \mathbb{R}_+$, we will set
$
D(p\|q) \triangleq p\log(p/q) - (p-q),
$
which gives the KL divergence between two Poisson probability distributions.
Second, we allow the KL divergence to act on matrices via the average KL divergence over their entries:
for two matrices $P$, $Q \in \mathbb{R}_+^{d_1 \times d_2}$, we define
\[
D(P\|Q) \triangleq \frac{1}{d_1 d_2}\sum_{i,j}D(P_{ij}\|Q_{ij}).
\]
For two probability distributions $\mathcal{P}$ and $\mathcal{Q}$ on a
countable set $A$, $d_H^2(\mathcal{P},\mathcal{Q})$ will denote the Hellinger distance
\[
d_H^2(\mathcal{P},\mathcal{Q}) = \sum_{x\in A} \left(\sqrt{\mathcal{P}(x)} - \sqrt{\mathcal{Q}(x)}\right)^2.
\]
Similarly, we abuse this notation slightly to denote the Hellinger distance between two Poisson distributions with different parameters (the arguments in the notation denote parameters of the Poisson distributions). We use the Hellinger distance between two Poisson distributions, which, for two scalars $p,q \in \mathbb{R}_+$, is given by,
$
d_H^2(p, q) \triangleq 2-2\exp\left\{-\frac{1}{2}\left(\sqrt{p}-\sqrt{q}\right)^2\right\}.
$
For matrices $P$, $Q \in \mathbb{R}_+^{d_1 \times d_2}$, the average Hellinger distance is defined by
\[
d_H^2(P,Q) \triangleq \frac{1}{d_1 d_2}\sum_{i,j}d_H^2(P_{ij},Q_{ij}).
\]

}
\subsection{Proof of Theorem \ref{maintheorem}}

To prove Theorem \ref{maintheorem}, the key will be to establish the concentration inequality (Lemma \ref{firstlemma}) and the lower bound for the average Hellinger distance (Lemma \ref{secondlemma}).

\begin{lemma}
Let $F_{\Omega, Y}(X)$ be the likelihood function defined in (\ref{likelihood}) and $\mathcal{S}$ be the set defined in (\ref{searchspace}), then
\begin{equation}
\begin{split}
&\mathbb{P} \left\{ \sup_{X\in \mathcal{S}} \left| F_{\Omega, Y}(X)-\mathbb{E}[F_{\Omega, Y}(X)]\right| \right.\\
& \quad \left. \geq C'\left( \alpha\sqrt{r}/\beta \right) \left( \alpha(e^2-2) + 3\log(d_1 d_2) \right) \cdot \right. \\
& \quad \left. \left(\sqrt{m(d_1+d_2)+d_1 d_2 \log(d_1 d_2)}\right) \right\}  \leq \frac{C}{d_1 d_2},
\end{split}
\label{resultfirstlemma}
\end{equation}
where $C'$ and $C$ are absolute positive constants and the probability and the expectation are both over the choice of $\Omega$ and the draw of $Y$.
\label{firstlemma}
\end{lemma}

\begin{lemma}
For any two matrices $P,Q \in \mathcal{S}$, we have
\[
d_H^2(P,Q) \geq \frac{1-e^{-T}}{4\alpha T} \frac{\|P-Q\|_F^2}{d_1 d_2},
\]
where $T=\frac{1}{8\beta}(\alpha-\beta)^2$.
\label{secondlemma}
\end{lemma}

We will prove Lemma \ref{firstlemma} and Lemma \ref{secondlemma} below, but first we use them in proving Theorem \ref{maintheorem}.

\begin{proof}[Proof of Theorem \ref{maintheorem}]
To begin, notice that for any choice of $X \in \mathcal{S}$,
    \begin{equation}
    \begin{aligned}
    &~~~~~\mathbb{E}\left[F_{\Omega,Y}(X) - F_{\Omega,Y}(M)\right] \\
    &= \frac{m}{d_1 d_2} \sum_{i,j}\left[ M_{ij}\log\left(\frac{X_{ij}}{M_{ij}}\right) - (X_{ij}-M_{ij})\right] \\
    &=-\frac{m}{d_1 d_2} \sum_{i,j}\left[  M_{ij}\log\left(\frac{M_{ij}}{X_{ij}}\right) - (M_{ij}-X_{ij}) \right] \\
    &=-\frac{m}{d_1 d_2} \sum_{i,j} D\left(M_{ij}\|X_{ij}\right) =-m D(M\|X),
    \end{aligned}
    \label{eqn50}
    \end{equation}
    where the expectation is over both $\Omega$ and $Y$.

    On the other hand, note that by assumption the true matrix $M \in \mathcal{S}$. Then for any $Z\in\mathcal{S}$, consider the difference below
    \begin{align}
    &~F_{\Omega,Y}(Z) - F_{\Omega,Y}(M) \nonumber \\
    =& ~F_{\Omega,Y}(Z) + \mathbb{E}[F_{\Omega,Y}(Z)] - \mathbb{E}[F_{\Omega,Y}(Z)]\nonumber \\
    & ~~ + \mathbb{E}[F_{\Omega,Y}(M)] - \mathbb{E}[F_{\Omega,Y}(M)] - F_{\Omega,Y}(M)\nonumber \\
    =& ~  \mathbb{E}[F_{\Omega,Y}(Z)]  - \mathbb{E}[F_{\Omega,Y}(M)] + \nonumber\\
    & ~~ F_{\Omega,Y}(Z)   - \mathbb{E}[F_{\Omega,Y}(Z)]  + \mathbb{E}[F_{\Omega,Y}(M)]  - F_{\Omega,Y}(M) \nonumber\\
    \leq&~ \mathbb{E}\left[ F_{\Omega,Y}(Z) - F_{\Omega,Y}(M) \right] +\nonumber \\
    & ~~ \left| F_{\Omega,Y}(Z)-\mathbb{E}[F_{\Omega,Y}(Z)]\right| + \left|F_{\Omega,Y}(M) - \mathbb{E}[F_{\Omega,Y}(M)]\right| \nonumber\\
    \leq& -m D(M\|Z) + 2\sup_{X\in \mathcal{S}} \left| F_{\Omega,Y}(X) - \mathbb{E}[F_{\Omega,Y}(X)]\right|,\label{Fchain}
    \end{align}
    where the second equality is to rearrange terms, the first inequality is due to triangle inequality, the last inequality is due to (\ref{eqn50}) and the fact that \[\left| F_{\Omega,Y}(Z)-\mathbb{E}[F_{\Omega,Y}(Z)]\right| \leq \sup_{X\in \mathcal{S}} \left| F_{\Omega,Y}(X) - \mathbb{E}[F_{\Omega,Y}(X)]\right|\] and  \[\left| F_{\Omega,Y}(M)-\mathbb{E}[F_{\Omega,Y}(M)]\right| \leq \sup_{X\in \mathcal{S}} \left| F_{\Omega,Y}(X) - \mathbb{E}[F_{\Omega,Y}(X)]\right|.\]
    Moreover, from the definition of $\widehat{M}$, we also have that $\widehat{M} \in \mathcal{S}$ and $F_{\Omega,Y}(\widehat{M}) \geq F_{\Omega,Y}(M) $. Thus, by substituting $\widehat{M}$ for $Z$ in (\ref{Fchain}), we obtain
    \[
    0\leq -m D(M\|\widehat{M}) + 2\sup_{X\in \mathcal{S}} \left| F_{\Omega,Y}(X) - \mathbb{E}[F_{\Omega,Y}(X)]\right|.
    \]
    To bound the second term in the above expression, we apply Lemma \ref{firstlemma}, and obtain that with probability at least $1-C /(d_1 d_2)$, we have
    \begin{equation*}
    \begin{split}
    0 &\leq -m D(M\|\widehat{M}) + 2C'\left( \alpha\sqrt{r}/\beta \right) \left( \alpha(e^2-2) + 3\log(d_1 d_2) \right) \cdot \\
    & \left(\sqrt{m(d_1+d_2)+d_1 d_2 \log(d_1 d_2)}\right).
    \end{split}
    \end{equation*}
    After rearranging terms, and use the fact that $\sqrt{d_1 d_2} \leq d_1 + d_2$, we obtain
    \begin{equation}
    \begin{split}
    D(M\|\widehat{M}) \leq & 2C'\left( \alpha\sqrt{r}/\beta \right) \left( \alpha(e^2-2) + 3\log(d_1 d_2) \right) \cdot \\
    & \left(\sqrt{\frac{d_1 +d_2}{m}} \sqrt{1+\frac{(d_1+d_2)\log(d_1 d_2)}{m}}\right).
    \label{theoremuse1}
    \end{split}
    \end{equation}
    Note that the KL divergence can be bounded below by the Hellinger distance (Chapter 3 in \cite{pollard2002user}). Using our notation to denote the parameters of the Poisson distributions in the argument of the distance, we have
    \begin{equation}
    d_H^2(p,q) \leq D(p\|q), \label{dD}
    \end{equation}
    for any two scalars $p,q \in \mathbb{R}_+$ that denote the parameters of the Poisson distributions.
    Thus, (\ref{theoremuse1}) together with (\ref{dD}) lead to
        \begin{equation}
    \begin{split}
    &d_H^2(M,\widehat{M}) \leq  2C'\left( \alpha\sqrt{r}/\beta \right) \left( \alpha(e^2-2) + 3\log(d_1 d_2) \right) \cdot \\
    &\quad \left(\sqrt{\frac{d_1 +d_2}{m}} \sqrt{1+\frac{(d_1+d_2)\log(d_1 d_2)}{m}}\right).
    \end{split}
    \end{equation}
    Finally, Theorem \ref{maintheorem} follows immediately from Lemma \ref{secondlemma}.

\end{proof}



Next, we will establish a tail bound for Poisson distribution with the method of establishing Chernoff bounds. And this result will be used for proving Lemma \ref{firstlemma}.

\begin{lemma}[Tail bound for Poisson]
    For $Y \sim \mbox{Poisson} (\lambda)$ with $\lambda \leq \alpha$, $\mathbb{P}(Y -\lambda \geq t) \leq e^{-t}$, for all $t \geq t_0$ where $t_0 \triangleq \alpha (e^2-3)$.
\label{extendedbernsteininequality}
\end{lemma}

\begin{proof}[Proof of Lemma \ref{extendedbernsteininequality}]
\yx{The proof below is a specialized version of Chernoff bound for Poisson random variable \cite{CannyCS174} when $\lambda$ is upper bounded by a constant.}
    For any $\theta \geq 0$, we have
    \begin{equation*}
    \begin{split}
        &~\mathbb{P}\left( Y -\lambda \geq t\right)
        =  \mathbb{P}\left( Y \geq t + \lambda \right)\\
        =&~\mathbb{P}\left(\theta Y \geq \theta\left(t+\lambda \right)\right)
        = \mathbb{P}\left(\exp\left(\theta Y \right) \geq \exp \left(\theta\left(t+ \lambda \right)\right) \right)\\
             \leq   &~
         \exp\left(-\theta\left(t+\lambda \right)\right) \mathbb{E}\left( e^{\theta Y}\right)
        = \exp(-\theta (\lambda+t)) \cdot \exp\left(\lambda (e^{\theta}-1)\right),
    \end{split}
    \end{equation*}
    where we have used Markov's inequality and the moment generating function for Poisson random variable above.
    mow let $\theta=2$, we have
    \begin{equation*}
    \begin{split}
        &\exp(t) \cdot \mathbb{P}\left( Y - \lambda \geq t\right) \leq \exp\left(-t + \lambda(e^2-3)\right).
    \end{split}
    \end{equation*}
    Given
    $
    t_0 \triangleq \alpha (e^2-3),
    $ then for all $t \geq t_0$, we have $\exp(t) \cdot \mathbb{P}\left( Y - \lambda \geq t\right)\leq 1$. It follows that $\mathbb{P}\left( Y - \lambda \geq t\right) \leq e^{-t}$ when
$
    t\geq  t_0
    \geq \lambda (e^2-3).
  $
\end{proof}

\begin{proof}[Proof of Lemma \ref{firstlemma}]

We begin by noting that for any $h>0$, by using Markov's inequality we have that

    \begin{equation}
    \begin{aligned}
    &\mathbb{P} \left\{ \sup_{X\in \mathcal{S}} \left| F_{\Omega, Y}(X)-\mathbb{E}[F_{\Omega, Y}(X)]\right| \right. \\
    & \quad \left. \geq C'\left( \alpha\sqrt{r}/\beta \right) \left( \alpha(e^2-2) + 3\log(d_1 d_2) \right) \cdot \right. \\
    & \quad \left. \left(\sqrt{m(d_1+d_2)+d_1 d_2 \log(d_1 d_2)}\right) \right\} \\
    =&~\mathbb{P} \left\{ \sup_{X\in \mathcal{S}} \left| F_{\Omega, Y}(X)-\mathbb{E}[F_{\Omega, Y}(X)\right]|^h \right. \\
    & \quad \left. \geq \left(C'\left( \alpha\sqrt{r}/\beta \right) \left( \alpha(e^2-2) + 3\log(d_1 d_2) \right) \cdot \right. \right.\\
    & \quad \left. \left. \left(\sqrt{m(d_1+d_2)+d_1 d_2 \log(d_1 d_2)}\right) \right)^h\right\} \\
    \leq &~ \mathbb{E}\left[\sup_{X\in \mathcal{S}} \left| F_{\Omega, Y}(X) - E[F_{\Omega, Y}(X)] \right|^h \right]/\\
    &
    \{\left(C'\left( \alpha\sqrt{r}/\beta \right) \left( \alpha(e^2-2) + 3\log(d_1 d_2) \right) \right.\cdot \\
    &\left.\left(\sqrt{m(d_1+d_2)+d_1 d_2 \log(d_1 d_2)}\right) \right)^h \}.
    \label{markov1}
    \end{aligned}
    \end{equation}
    The bound in (\ref{resultfirstlemma}) will follow by combining this with an upper bound on $\mathbb{E}\left[\sup_{X\in \mathcal{S}} \left| F_{\Omega, Y}(X) - E[F_{\Omega, Y}(X)] \right|^h \right]$ and setting $h=\log(d_1d_2)$.

Let
$\epsilon_{ij}$s be i.i.d. Rademacher random variables.
    In the following derivation, the first inequality is due the Radamacher symmetrization argument (Lemma 6.3 in \cite{ledoux1991probability}) and the second inequality is due to the power mean inequality: $(a+b)^h \leq 2^{h-1}(a^h+b^h)$ if $a,b>0$ and $h\geq 1$. Then we have
    \begin{equation}
        \begin{split}
        &\quad \mathbb{E}\left[\sup_{X\in \mathcal{S}} \left| F_{\Omega, Y}(X) - \mathbb{E}F_{\Omega, Y}(X) \right|^h \right] \\
        &\leq 2^h \mathbb{E}\left[ \sup_{X\in \mathcal{S}} \left| \sum_{i,j} \epsilon_{ij} \mathbb{I}\{[(i,j)\in \Omega]\} (Y_{ij}\log X_{ij} - X_{ij}) \right|^h \right]
        \end{split}\nonumber
        \end{equation}
        %
        \begin{equation}
        \begin{split}
        &\leq 2^h \mathbb{E}\left[ 2^{h-1} \left( \sup_{X\in \mathcal{S}} \left| \sum_{i,j} \epsilon_{ij} \mathbb{I}\{[(i,j)\in \Omega]\} (Y_{ij}(-\log X_{ij})) \right|^h \right) \right. \\
        & \quad + 2^{h-1} \left.\left( \sup_{X\in \mathcal{S}} \left| \sum_{i,j} \epsilon_{ij} \mathbb{I}\{[(i,j)\in \Omega]\} X_{ij} \right|^h \right) \right]\\
        %
        & = 2^{2h-1}  \mathbb{E}\left[ \sup_{X\in \mathcal{S}} \left| \sum_{i,j} \epsilon_{ij} \mathbb{I}\{[(i,j)\in \Omega]\} (Y_{ij}(-\log X_{ij})) \right|^h \right] \\
        & \quad + 2^{2h-1} \mathbb{E}\left[ \sup_{X\in \mathcal{S}} \left| \sum_{i,j} \epsilon_{ij} \mathbb{I}\{[(i,j)\in \Omega]\} X_{ij} \right|^h \right],
    \end{split}
    \label{wholepart}
    \end{equation}
    where the expectation are over both $\Omega$ and $Y$.

    To bound the first term of (\ref{wholepart}) with the assumption that $\|X\|_* \leq \alpha \sqrt{rd_1 d_2}$, we use a contraction principle (Theorem 4.12 in \cite{ledoux1991probability}). We let $\phi(t)=-\beta \log(t+1)$. We know $\phi(0)=0$ and $|\phi^{'}(t)|=|\beta/(t+1)|$, so $|\phi^{'}(t)| \leq 1$ if $t \geq \beta-1$. Setting $Z=X-\textbf{1}_{d_1\times d_2}$, then we have $Z_{ij} \geq \beta-1, \forall (i,j) \in \llbracket d_1\rrbracket\times\llbracket d_2\rrbracket$ and $\|Z\|_* \leq \alpha\sqrt{rd_1 d_2} + \sqrt{d_1 d_2}$ by triangle inequality. Therefore, $\phi(Z_{ij})$ is a contraction and it vanishes at $0$. We obtain that
    \begin{equation}
    \begin{split}
        &\quad ~2^{2h-1} \mathbb{E}\left[ \sup_{X\in \mathcal{S}} \left| \sum_{i,j} \epsilon_{ij} \mathbb{I}\{[(i,j)\in \Omega]\} (Y_{ij}(-\log X_{ij})) \right|^h \right] \\
        &\leq 2^{2h-1} \mathbb{E}\left[\max_{i,j} Y_{ij}^h \right] \cdot\\
        &\quad \quad\mathbb{E}\left[ \sup_{X\in \mathcal{S}} \left| \sum_{i,j} \epsilon_{ij} \mathbb{I}\{[(i,j)\in \Omega]\} ((-\log X_{ij})) \right|^h \right] \\
        &=2^{2h-1} \mathbb{E}\left[\max_{i,j} Y_{ij}^h \right] \cdot \\
        &\quad\quad \mathbb{E}\left[ \sup_{X\in \mathcal{S}} \left| \sum_{i,j} \epsilon_{ij} \mathbb{I}\{[(i,j)\in \Omega]\} \left(\frac{1}{\beta}\phi(Z_{ij})\right) \right|^h \right] \\
        &\leq 2^{2h-1}\left(\frac{2}{\beta}\right)^h \mathbb{E}\left[\max_{i,j} Y_{ij}^h \right]\cdot\\
        &\quad\quad \mathbb{E}\left[ \sup_{X\in \mathcal{S}} \left| \sum_{i,j} \epsilon_{ij} \mathbb{I}\{[(i,j)\in \Omega]\} Z_{ij}) \right|^h \right] \\
        &=2^{2h-1}\left(\frac{2}{\beta}\right)^h \mathbb{E}\left[\max_{i,j} Y_{ij}^h \right] \mathbb{E}\left[ \sup_{X\in \mathcal{S}}\left| \langle \Delta_{\Omega} \circ E , Z \rangle \right|^h \right],
    \end{split}
    \end{equation}
    where $E$ denotes the matrix with entries given by $\epsilon_{ij}$, $\Delta_{\Omega}$ denotes the indicator matrix for $\Omega$ and $\circ$ denotes the Hadamard product.

    The dual norm of spectral norm is nuclear norm. Using the H\"{o}lder's inequality for Schatten norms in \cite{Watrous2011}, which is, $|\langle A,B \rangle| \leq \|A\|\|B\|_*$, we have

    \begin{equation}
    \begin{split}
        &~\quad 2^{2h-1} \mathbb{E}\left[ \sup_{X\in \mathcal{S}} \left| \sum_{i,j} \epsilon_{ij} \mathbb{I}\{[(i,j)\in \Omega]\} (Y_{ij}(-\log X_{ij})) \right|^h \right] \\
        &\leq 2^{2h-1}\left(\frac{2}{\beta}\right)^h \mathbb{E}\left[\max_{i,j} Y_{ij}^h \right] \mathbb{E}\left[ \sup_{X \in \mathcal{S}} \|E \circ \Delta_{\Omega} \|^h \|Z\|_*^h \right] \\
        &\leq2^{2h-1}\left(\frac{2}{\beta}\right)^h \left(\alpha\sqrt{r}+1\right)^h \left(\sqrt{d_1 d_2 }\right)^h \mathbb{E}\left[\max_{i,j} Y_{ij}^h \right] \cdot \\
        &\qquad \mathbb{E} \left[\|E \circ \Delta_{\Omega} \|^h\right],
    \end{split}
    \label{firstterm}
    \end{equation}

    Similarly, the second term of (\ref{wholepart}) can be bounded as follows:
    \begin{equation}
    \begin{split}
    &2^{2h-1} \mathbb{E}\left[ \sup_{X\in \mathcal{S}} \left| \sum_{i,j} \epsilon_{ij} \mathbb{I}\{[(i,j)\in \Omega]\} X_{ij} \right|^h \right] \\
    \leq &~2^{2h-1} \mathbb{E}\left[ \sup_{X \in \mathcal{S}} \|E \circ \Delta_{\Omega} \|^h \|X\|_*^h \right] \\
    \leq &~2^{2h-1} \left(\alpha \sqrt{r}\right)^h \left(\sqrt{d_1 d_2}\right)^h \mathbb{E}\left[ \|E \circ \Delta_{\Omega} \|^h \right].
    \end{split}
    \label{secondterm}
    \end{equation}

    Plugging (\ref{firstterm}) and (\ref{secondterm}) into (\ref{wholepart}), we have
    \begin{equation}
    \begin{split}
    &\mathbb{E}\left[\sup_{X\in \mathcal{S}} \left| F_{\Omega, Y}(X) - \mathbb{E}F_{\Omega, Y}(X) \right|^h \right] \\
    \leq &~2^{2h-1} \left(\alpha \sqrt{r}+1\right)^h \left(\sqrt{d_1 d_2}\right)^h \mathbb{E}\left[ \|E \circ \Delta_{\Omega} \|^h \right] \cdot\\
    & \left( \left(\frac{2}{\beta}\right)^h \mathbb{E}\left[\max_{i,j} Y_{ij}^h \right] +1 \right).
    \label{secondderivation}
    \end{split}
    \end{equation}

    To bound $\mathbb{E} \left[\|E \circ \Delta_{\Omega} \|^h\right]$, we use the very first inequality on Page 215 of \cite{davenport20121}:
    \begin{equation}
    \begin{split}
    &\mathbb{E} \left[\|E \circ \Delta_{\Omega} \|^h\right] \\
    \leq& ~C_0 \left(2(1+\sqrt{6})\right)^h \left( \sqrt{\frac{m(d_1+d_2)+d_1 d_2 \log(d_1 d_2)}{d_1 d_2}} \right)^h
    \end{split}\nonumber
    \end{equation}
    for some constant $C_0$. Therefore, the only term we need to bound is $\mathbb{E}\left[ \max_{i,j} Y_{ij}^{h}\right] $.

    From Lemma \ref{extendedbernsteininequality}, if $t \geq t_0$, then for any $(i,j) \in \llbracket d_1\rrbracket\times \llbracket d_2\rrbracket$, the following inequality holds since $t_0 > \alpha$:
    \begin{equation}
    \begin{split}
   & ~\mathbb{P}\left( \left| Y_{ij} - M_{ij} \right| \geq t \right)\\
    = & ~\mathbb{P}\left( Y_{ij} \geq M_{ij} +  t \right) + \mathbb{P}\left( Y_{ij} \leq M_{ij} -  t \right) \\
    \leq & ~\exp(-t) + 0
     = \mathbb{P}(W_{ij} \geq t),
    \end{split}
    \end{equation}
    where $W_{ij}$s are independent standard exponential random variables.
    Because $|Y_{ij}-M_{ij}|$s and $W_{ij}$'s are all non-negative random variables and $\max(x_1,x_2,\ldots,x_n)$ is an increasing function defined on $\mathbb{R}^n$, we have, for any $h\geq 1$,
    \begin{equation}
    \mathbb{P} \left( \max_{i,j} \left| Y_{ij} - M_{ij} \right|^h \geq t \right) \leq  \mathbb{P}(\max_{i,j} W_{ij}^h \geq t),
    \label{exponentialappro}
    \end{equation}
    for any $t\geq (t_0)^h$.

    Below we use the fact that for any positive random variable $q$, we can write
    $
    \mathbb{E}[q] = \int_{0}^{\infty}\mathbb{P}(q\geq t) dt,
    $
    and then
    \begin{equation}
    \begin{split}
    &\quad~ \mathbb{E}\left[ \max_{i,j} Y_{ij}^{h}\right] \\
    & \leq 2^{2h-1} \left( \alpha^h + \mathbb{E}\left[ \max_{i,j} \left| Y_{ij}-M_{ij} \right|^{h}\right] \right) \\
    & = 2^{2h-1} \left( \alpha^h + \int_{0}^{\infty} \mathbb{P} \left( \max_{i,j} \left| Y_{ij}-M_{ij} \right|^{h} \geq t \right) dt \right) \\
    & \leq 2^{2h-1} \left( \alpha^h + (t_0)^h + \int_{(t_0)^h}^{\infty} \mathbb{P} \left( \max_{i,j} \left| Y_{ij}-M_{ij} \right|^{h} \geq t \right) dt \right) \\
    & \leq 2^{2h-1} \left( \alpha^h + (t_0)^h + \int_{(t_0)^h}^{\infty} \mathbb{P} \left( \max_{i,j} W_{ij}^{h} \geq t \right) dt \right) \\
    & \leq 2^{2h-1} \left( \alpha^h + (t_0)^h + \mathbb{E} \left[ \max_{i,j} W_{ij}^{h}\right] \right)
    \end{split}
    \end{equation}

      Above, firstly we use triangle inequality and power mean inequality, then along with independence, we use (\ref{exponentialappro}) in the third inequality. By standard computations for exponential random variables,
    \begin{equation}
    \mathbb{E} \left[ \max_{i,j} W_{ij}^h\right] \leq 2h! + \log^{h}(d_1 d_2).
    \end{equation}
    Thus, we have
    \begin{equation}
    \begin{split}
    &\mathbb{E}\left[ \max_{i,j} Y_{ij}^{h}\right] \leq 2^{2h-1} \left( \alpha^h + (t_0)^h + 2h! + \log^{h}(d_1 d_2) \right).
    \label{onlyterm}
    \end{split}
    \end{equation}

    Therefore, combining (\ref{onlyterm}) and (\ref{secondderivation}), we have
    \begin{equation}
    \begin{split}
    &\mathbb{E}\left[\sup_{X\in \mathcal{S}} \left| F_{\Omega, Y}(X) - \mathbb{E}[F_{\Omega, Y}(X)] \right|^h \right] \\
    \leq &~2^{4h-1} \left(\alpha \sqrt{r}+1\right)^h \left(\sqrt{d_1 d_2}\right)^h \mathbb{E}\left[ \|E \circ \Delta_{\Omega} \|^h \right] \cdot\\
    &\left(\frac{2}{\beta}\right)^h \left( \alpha^h + (t_0)^h + 2h! + \log^{h}(d_1 d_2) \right).
    \end{split}
    \end{equation}
    Then,
    \begin{equation}
    \begin{split}
    &\left(\mathbb{E}\left[\sup_{X\in \mathcal{S}} \left| F_{\Omega, Y}(X) - \mathbb{E}[F_{\Omega, Y}(X)] \right|^h \right]\right)^{\frac{1}{h}} \\
    \leq &~16 \left(\alpha \sqrt{r}+1\right) \left(\sqrt{d_1 d_2}\right) \mathbb{E}\left[ \|E \circ \Delta_{\Omega} \|^h \right]^{\frac{1}{h}} \cdot\\
    &~\left(\frac{2}{\beta}\right) \left( \alpha + t_0 + 2h + \log(d_1 d_2) \right) \\
    \leq &~16 \left(\frac{2}{\beta}\right) \left(\alpha \sqrt{r}+1\right) \left(\sqrt{d_1 d_2}\right) \mathbb{E}\left[ \|E \circ \Delta_{\Omega} \|^h \right]^{\frac{1}{h}} \cdot\\
    &~ \left( \alpha(e^2-2) + 3\log(d_1 d_2) \right) \\
    \leq &~128\left(1+\sqrt{6}\right) C_0^{\frac{1}{h}} \left(\frac{\alpha\sqrt{r}}{\beta} \right) \left( \alpha(e^2-2) + 3\log(d_1 d_2) \right) \cdot\\
    &  ~\left(\sqrt{m(d_1+d_2)+d_1 d_2 \log(d_1 d_2)}\right).
    \end{split}
    \end{equation}
    where we use the fact that $(a^h + b^h + c^h + d^h)^{1/h} \leq a+b+c+d$ if $a,b,c,d>0$ in the first inequality and we take $h=\log(d_1 d_2)\geq 1$ in the second and the third inequality.

    Plugging this into (\ref{markov1}), we obtain that the probability in (\ref{markov1}) is upper bounded by
    \[
    C_0\left( \frac{128(1+\sqrt{6})}{C'} \right)^{\log(d_1 d_2)} \leq \frac{C_0}{d_1 d_2},
    \]
    provided that $C'\geq 128\left(1+\sqrt{6}\right)e$, which establishes this lemma.

\end{proof}

\begin{proof}[Proof of Lemma \ref{secondlemma}]
Assuming $x$ is any entry in $P$ and $y$ is any entry in $Q$, then $\beta \leq x,y \leq \alpha$ and $0 \leq |x-y| \leq \alpha-\beta$.
By the mean value theorem there exists an $\xi(x,y)\in [\beta,\alpha]$ such that
\begin{align*}
 \frac{1}{2} (\sqrt{x}-\sqrt{y})^2 & = \frac{1}{2} \left(\frac{1}{2\sqrt{\xi(x,y)}}(x-y)\right)^2 \\
&= \frac{1}{8\xi(x,y)}(x-y)^2 \leq T.
\end{align*}
The function $f(z)=1-e^{-z}$ is concave in $[0,+\infty]$, so if $z\in [0,T]$, we may bound it from below with a linear function
\begin{equation}
1-e^{-z} \geq \frac{1-e^{-T}}{T}z.
\label{lemma2use}
\end{equation}
Plugging $z=\frac{1}{2} (\sqrt{x}-\sqrt{y})^2 = \frac{1}{8\xi(x,y)}(x-y)^2 $ into (\ref{lemma2use}), we have
\begin{equation}
\begin{split}
&2-2\exp\left(-\frac{1}{2} (\sqrt{x}-\sqrt{y})^2 \right) \geq \frac{1-e^{-T}}{T} \frac{1}{4\xi(x,y)}(x-y)^2 \\
&\geq \frac{1-e^{-T}}{T} \frac{1}{4\alpha}(x-y)^2.
\label{lemma2use2}
\end{split}
\end{equation}
Note that (\ref{lemma2use2}) holds for any $x$ and $y$. This concludes the proof.
\end{proof}

\subsection{Proof of Theorem \ref{maintheorem2}}
Before providing the proof, we first establish two useful lemmas. First, we consider the construction of the set $\chi$.

\begin{lemma}[Lemma A.3 in \cite{davenport20121}]
Let \[H \triangleq \left\{ X : \|X\|_* \leq \alpha \sqrt{r d_1 d_2}, \|X\|_{\infty} \leq \alpha \right\}\] and $\gamma \leq 1$ be such that $r/\gamma^2$ is an integer.
Suppose  $r/\gamma^2 \leq d_1$, then we may construct a set $\chi \in H$ of size
\begin{equation}
|\chi| \geq \exp\left( \frac{r d_2}{16\gamma^2} \right)
\label{packingsetsize}
\end{equation}
with the following properties:
\begin{enumerate}
\item For all $X \in \chi$, each entry has $|X_{ij}| = \alpha \gamma$.
\item For all $X^{(i)}$,$X^{(j)} \in \chi$, $i\neq j$,
\[
\|X^{(i)} - X^{(j)} \|_F^2 > \alpha^2 \gamma^2 d_1 d_2/2.
\]
\end{enumerate}
\label{packingset}
\end{lemma}

Second, we consider about the KL divergence.

\begin{lemma}
For $x,y >0$,
$
D(x\|y) \leq (y-x)^2/y.
$
\label{KLdivergence}
\end{lemma}
\begin{proof}[Proof of Lemma \ref{KLdivergence}]
First assume $x\leq y$. Let $z=y-x$. Then $z \geq 0$ and
$
D(x\|x+z) = x\log \frac{x}{x+z} + z.
$
Taking the first derivative of this with respect to $z$, we have
$
\frac{\partial}{\partial z} D(x\|x+z) = \frac{z}{x+z}.
$
Thus, by Taylor's theorem, there is some $\xi \in [0,z]$ so that
$
D(x\|y) = D(x\|x) + z \cdot \frac{\xi}{x+\xi}.
$
Since the $z\xi/(x+\xi)$ increases in $\xi$, we may replace $\xi$ with $z$ and obtain
$
D(x\|y) \leq \frac{(y-x)^2}{y}.
$
For $x>y$, with the similar argument we may conclude that for $z=y-x<0$ there is some $\xi \in [z,0]$ so that
$
D(x\|y) = D(x\|x) + z \cdot \frac{\xi}{x+\xi}.
$
Since $z<0$ and $\xi / (x+\xi)$ increases in $\xi$, then $z\xi/(x+\xi)$ decreases in $\xi$. We may also replace $\xi$ with $z$ and this proves the lemma.
\end{proof}

Next, we will show how Lemma \ref{packingset} and Lemma \ref{KLdivergence} imply Theorem \ref{maintheorem2}. We will prove the theorem by contradiction.

\begin{proof}[Proof of Theorem \ref{maintheorem2}]

Without loss of generality,  assume $d_2 \geq d_1$.
We choose $\epsilon > 0$ such that
\begin{equation}
\epsilon^2 = \min\left\{ \frac{1}{256}, C_2 \alpha^{3/2} \sqrt{\frac{rd_2}{m}}\right\},
\label{epsilon1}
\end{equation}
where $C_2$ is an absolute constant that will be be specified later.
We will next use Lemma \ref{packingset} to construct a set $\chi$, choosing $\gamma$ such that $r/\gamma^2$ is an integer and
$$
\frac{4\sqrt{2}\epsilon}{\alpha} \leq \gamma \leq \frac{8\epsilon}{\alpha}.
$$
We can make such a choice because
$$
\frac{\alpha^2 r}{64\epsilon^2} \leq \frac{r}{\gamma^2} \leq \frac{\alpha^2 r}{32\epsilon^2}
$$
and
$$
\frac{\alpha^2 r}{32\epsilon^2} - \frac{\alpha^2 r}{64\epsilon^2}  = \frac{\alpha^2 r}{64\epsilon^2} > 4\alpha^2 r > 1.
$$
We verify that such a choice for $\gamma$ satisfies the the requirements of Lemma \ref{packingset}. Indeed, since $\epsilon \leq \frac{1}{16}$ and $\alpha \geq 1$, we have $\gamma \leq \frac{1}{2}< 1$.
Further, we assume in the theorem that the right-hand side of (\ref{epsilon1}) is larger than $C_1 r\alpha^2/d_1$, which implies $r/\gamma^2 \leq d_1$ for an appropriate choice of $C_1$.

Let $\chi'_{\alpha/2, \gamma}$ be the set whose existence is guaranteed in Lemma \ref{packingset}, with this choice of $\gamma$ and with $\alpha/2$ instead of $\alpha$. Then we can construct $\chi$  by defining
$$
\chi \triangleq \left\{ X' + \alpha\left(1-\frac{\gamma}{2}\right) \textbf{1}_{d_1 \times d_2} : X' \in \chi'_{\alpha/2, \gamma} \right\},
$$
where $\textbf{1}_{d_1 \times d_2}$ denotes an $d_1$-by-$d_2$ matrix of all ones.
Note that $\chi$ has the same size as $\chi'_{\alpha/2, \gamma}$, i.e.$|\chi|$ satisfies (\ref{packingsetsize}). $\chi$ also has the same bound on pairwise
distances
\begin{equation}
\|X^{(i)} - X^{(j)} \|_F^2 \geq \frac{\alpha^2}{4}\frac{\gamma^2 d_1 d_2}{2} \geq 4 d_1 d_2 \epsilon^2,
\label{distance1}
\end{equation}
for any two matrices $X^{(i)}, X^{(j)} \in \chi$.
Define $\alpha' \triangleq (1-\gamma)\alpha$, then every entry of $X \in \chi$ has $X_{ij} \in \{\alpha, \alpha'\}$. Since we assume $r \geq 4$ in theorem statement, for any $X \in \chi$, we have that for some $X' \in \chi'_{\alpha/2, \gamma}$,
\begin{align*}
\|X\|_* &= \| X' + \alpha\left(1-\frac{\gamma}{2}\right) \textbf{1}_{d_1 \times d_2}\|_* \\
&\leq \frac{\alpha}{2} \sqrt{r d_1 d_2} + \alpha \sqrt{d_1 d_2}
\leq \alpha \sqrt{r d_1 d_2}.
\end{align*}
%
Since the $\gamma$ we choose is less than $1/2$, we have that $\alpha'$ is greater than $\alpha/2$. Therefore, from the assumption that $\beta \leq \alpha/2$, we conclude that $\chi \subset \mathcal{S}$.

Now suppose for the sake of a contradiction that there exists an algorithm such that for any $X \in \mathcal{S}$, when given access to the measurements on $\Omega_0$, returns $\widehat{X}$ such that
\begin{equation}
\frac{1}{d_1 d_2} \|X-\widehat{X}\|_F^2 < \epsilon^2
\label{assumption}
\end{equation}
with probability at least $1/4$. We will imagine running this algorithm on a matrix $X$ chosen uniformly at random from $\chi$. Let
$$
X^* = \arg \min_{Z \in \chi} \|Z- \widehat{X}\|_F^2.
$$
By the same argument as that in \cite{davenport20121}, we can claim that $X^* = X$ as long as (\ref{assumption}) holds.
Indeed, for any $X' \in \chi$ with $X' \neq X$, from (\ref{distance1}) and (\ref{assumption}), we have that
\[
\|X'-\widehat{X}\|_F \geq \|X'-X\|_F-\|X-\widehat{X}\|_F > \sqrt{d_1 d_2} \epsilon.
\]
At the same time, since $X \in \chi$ is a candidate for $X^*$, we have that
\[
\|X^*-\widehat{X}\|_F \leq \|X-\widehat{X}\|_F \leq \sqrt{d_1 d_2} \epsilon.
\]
Thus, if (\ref{assumption}) holds, then $\|X^*-\widehat{X}\|_F < \|X'-\widehat{X}\|_F$ for any $X' \in \chi$ with $X' \neq X$, and hence we must
have $X^*=X$.

Using the assumption that (\ref{assumption}) holds with probability at least
$1/4$, we have that
\begin{equation}
\mathbb{P} (X^* \neq X) \leq \frac{3}{4}.
\label{inequality1}
\end{equation}
We will show that this probability must in fact be large, generating our contradiction.

By a variant of Fano's inequality in \cite{yu1997assouad}, we have
\begin{equation}
\mathbb{P} (X^* \neq X) \geq 1- \frac{\max_{X^{(k)}\neq X^{(l)}} \widetilde{D}(X^{(k)} \| X^{(l)}) +1}{\log |\chi|},
\label{inequality2}
\end{equation}
where \[\widetilde{D}(X^{(k)} \| X^{(l)}) \triangleq \sum_{(i,j)\in \Omega_0}D(X_{ij}^{(k)} \| X_{ij}^{(l)}),\] and the maximum is taken over all pairs of different matrices $X^{(k)}$ and $X^{(l)}$ in $\chi$.
%
For any such pairs $X^{(k)}, X^{(l)} \in \chi$, we know that $D(X_{ij}^{(k)}\|X_{ij}^{(l)})$ is either $0$, $D(\alpha\|\alpha')$, or $D(\alpha'\|\alpha)$ for $(i,j) \in \llbracket d_1 \rrbracket \times \llbracket d_2 \rrbracket$. Define an upper bound on the KL divergence quantities
\[
D \triangleq \max_{X^{(k)}\neq X^{(l)}} \widetilde{D}(X^{(k)}\| X^{(l)}).
\]
By the assumption that $|\Omega_0|=m$, using Lemma \ref{KLdivergence} and the fact that $\alpha' < \alpha$, we have
$$
D \leq \frac{m(\gamma \alpha)^2}{\alpha'} \leq \frac{64 m\epsilon^2}{\alpha'}.
$$
Combining (\ref{inequality1}) and (\ref{inequality2}), we have that
\begin{equation}
\begin{split}
\frac{1}{4} &\leq 1-\mathbb{P}(X \neq X^*) \leq \frac{D+1}{\log |\chi|} \\
&\leq 16\gamma^2 \left(\frac{\frac{64 m\epsilon^2}{\alpha'}+1}{rd_2} \right) \leq 1024\epsilon^2\left(\frac{\frac{64 m\epsilon^2}{\alpha'}+1}{\alpha^2 rd_2} \right).
\end{split}
\label{contradiction}
\end{equation}

We now show that for appropriate values of $C_0$ and $C_2$, this leads to a contradiction.
Suppose $64m\epsilon^2 \leq \alpha'$, then with (\ref{contradiction}), we have that
$$
\frac{1}{4} \leq 1024 \epsilon^2 \frac{2}{\alpha^2 r d_2},
$$
which together with (\ref{epsilon1}) implies that $\alpha^2 r d_2 \leq 32$. If we set $C_0>32$, then this would lead to a contradiction.
Next, suppose $64m\epsilon^2 > \alpha'$, then (\ref{contradiction}) simplifies to
$$
\frac{1}{4} < 1024\epsilon^2 \left( \frac{128m\epsilon^2}{(1-\gamma)\alpha^3 rd_2} \right).
$$
Since $1-\gamma > 1/2$, we have
$$
\epsilon^2 > \frac{\alpha^{3/2}}{1024}\sqrt{\frac{rd_2}{m}}.
$$
Setting $C_2 \leq 1/1024$ in (\ref{epsilon1}) leads to a contradiction. Therefore, (\ref{assumption}) must be incorrect with probability at least $3/4$, which proves the theorem.

\end{proof}

\section{Proofs of Lemma \ref{convergence1} and Lemma \ref{convergence2}}

\begin{lemma}
If $f$ is a closed convex function satisfying Lipschitz condition (\ref{Lipschitz}), then for any $X,Y \in \mathcal{S}$, the following inequality holds:
$$
f(Y) \leq f(X) + \langle \nabla f(X), Y-X \rangle + \frac{L}{2} \|Y-X\|_F^2.
$$
\label{UPlip}
\end{lemma}

\begin{proof}
Let $Z$ = $Y-X$, then we have
\begin{equation}
\begin{split}
 f(Y) = &~ f(X) +  \langle \nabla f(X), Z \rangle \\
 & ~+ \int_0^1 \langle \nabla f(X+tV)- \nabla f(X), Z \rangle ~dt \\
\leq &~ f(X) +  \langle \nabla f(X), Z \rangle \\
& ~ + \int_0^1 \|f(X+tV)- \nabla f(X)\|_F\|Z\|_F~ dt \\
\leq &~ f(X) +  \langle \nabla f(X), Z \rangle + \int_0^1 Lt \|Z\|_F^2~ dt \\
= & ~ f(X) + \langle \nabla f(X), Y-X \rangle + \frac{L}{2} \|Y-X\|_F^2,
\end{split} \nonumber
\end{equation}
where we use Taylor expansion with integral remainder in the first line, the fact that dual norm of Frobenius norm is itself in the second line and Lipschitz condition
in the third line.
\end{proof}


In the following, proofs for Lemma \ref{convergence1} and Lemma \ref{convergence2} use results from \cite{Ghaoui2010}.

\begin{proof}[Proof of Lemma \ref{convergence1}]
As is well known, proximal mapping of a $Y \in \mathcal{S}$ associated with a closed convex function $h$ is given by
$$
\mbox{prox}_{th}(Y) \triangleq \arg \min_{X} \left( t \cdot h(X) + \frac{1}{2}\| X-Y \|_F^2 \right),
$$
where $t>0$ is a multiplier. In our case, $h(P) = \mathbb I_{\mathcal{S}}(P)$. Define for each $P \in \mathcal{S}$ that
$$
G_t(P) \triangleq \frac{1}{t} \left( P - \mbox{prox}_{th} \left(P-t\nabla f(P)\right) \right),
$$
then we can know by the characterization of subgradient that
\begin{equation}
G_t(P) - \nabla f(P) \in \partial h(P),
\label{subgradient}
\end{equation}
where $\partial h(P)$ is the subdifferential of $h$ at $P$.
Noticing that $P-tG_t(P) \in \mathcal{S}$, then from Lemma \ref{UPlip} we have
\begin{equation}
f(P-t G_t(P)) \leq f(P) - \langle \nabla f(P), t G_t(P) \rangle+ \frac{t}{2} \|G_t(P)\|_F^2,
\label{UPlip2}
\end{equation}
for all $0\leq t \leq 1/L$.
Define that $g(P) \triangleq f(P) + h(P)$. Combining
(\ref{subgradient}) and (\ref{UPlip2}) and using the fact that $f$ and $h$ are convex functions, we have for any $Z\in \mathcal{S}$ and $0\leq t \leq 1/L$,
\begin{equation}
g(P-tG_t(P)) \leq g(Z) + \langle G_t(P), P-Z \rangle - \frac{t}{2} \|G_t(P)\|_F^2,
\label{eq1}
\end{equation}
which is an analog to inequality $(3.3)$ in \cite{Ghaoui2010}.
Taking $Z=\widehat{M}$ and $P=X_k$ in (\ref{eq1}), then we have for any $k \geq 0$,
\begin{equation}
\begin{split}
g(X_{k+1}) - g(\widehat{M}) \leq &~ \langle G_t(X_{k}), X_k - \widehat{M} \rangle - \frac{t}{2} \|G_t(X_k)\|_F^2 \\
= & ~\frac{1}{2t} \left( \|X_k - \widehat{M}\|_F^2 - \|X_{k+1}-\widehat{M}\|_F^2 \right),
\end{split}
\label{oneiteration}
\end{equation}
where we use the fact that $\langle P,P \rangle = \|P\|_F^2$.
By taking $Z = X_k$ and $P=X_k$ in (\ref{eq1}) we know that $g(X_{k+1}) < g(X_k)$ for any $k \geq 0$. Thus, taking $t=1/L$, we have,
\begin{equation}
\begin{split}
& g(X_k) - g(\widehat{M})
\leq  \frac{1}{k} \sum_{i=0}^{k-1} \left( g(X_{i+1}) - g(\widehat{M}) \right) \\
\leq & ~ \frac{L}{2k} \sum_{i=0}^{k-1} \left(\|X_i - \widehat{M}\|_F^2 - \|X_{i+1}-\widehat{M}\|_F^2  \right)
\leq   \frac{L \|X_0 - \widehat{M}\|_F^2}{2k}.
\end{split}
\end{equation}
Since $X_k \in \mathcal{S}$ for any $k\geq 0$ and $\widehat{M} \in \mathcal{S}$, we have that $h(X_k) =0$ for any $k \geq 0$ and $h(\widehat{M})=0$, which completes the proof.

\end{proof}

\begin{proof}[Proof of Lemma \ref{convergence2}]
The definitions and notations in the proof of Lemma \ref{convergence1} are also valid in the proof of Lemma \ref{convergence2}.

Define that $V_0 \triangleq X_0$ and for any $k\geq 1$,
$$
a_k \triangleq \frac{2}{k+1}, ~V_k \triangleq X_{k-1} + \frac{1}{a_k}\left(X_{k} - X_{k-1}\right).
$$
For any  $0\leq t\leq 1/L$, noticing that
$$
X_k = Z_{k-1} - tG_t(Z_{k-1}),
$$
then we can rewrite $V_k$ as
$$
V_k = V_{k-1} - \frac{t}{a_k}G_t(Z_{k-1}).
$$
Taking $Z=X_{k-1}$ and $Z=\widehat{M}$ in (\ref{eq1}) and making convex combination we have
\begin{equation}
\begin{split}
g(X_k) \leq &~ (1-a_k)g(X_{k-1})+a_k g(\widehat{M}) \\
            &+a_k \langle G_t(Z_{k-1}), V_{k-1}-\widehat{M} \rangle - \frac{t}{2} \|G_t(Z_{k-1})\|_F^2 \\
=& ~(1-a_k)g(X_{k-1})+a_k g(\widehat{M}) \\
&~+ \frac{a_k^2}{2t} \left( \|V_{k-1} - \widehat{M}\|_F^2 - \|V_{k}-\widehat{M}\|_F^2 \right).
\end{split}
\end{equation}
After rearranging terms, we have
\begin{equation}
\begin{split}
&\frac{1}{a_k^2}(g(X_{k})-g(\widehat{M}))+\frac{1}{2t}\|V_k-\widehat{M}\|_F^2 \leq \\
&\frac{1-a_k}{a_k^2}(g(X_{k-1})-g(\widehat{M}))+\frac{1}{2t}\|V_{k-1}-\widehat{M}\|_F^2.
\end{split}
\label{leq1}
\end{equation}
Notice that $(1-a_k)/(a_k^2) \leq 1/(a_{k-1}^2)$ for any $k\geq 1$. Applying inequality (\ref{leq1}) recursively, we have,
\begin{equation}
\frac{1}{a_k^2}(g(X_{k})-g(\widehat{M}))+\frac{1}{2t}\|V_k-\widehat{M}\|_F^2 \leq \frac{1}{2t} \|X_0 - \widehat{M} \|_F^2.
\end{equation}
Taking $t=1/L$, we have
$$
g(X_k) - g(\widehat{M}) \leq \frac{2L \|X_0 - \widehat{M}\|_F^2}{(k+1)^2}.
$$
Since $X_k \in \mathcal{S}$ for any $k\geq 0$ and $\widehat{M} \in \mathcal{S}$, we have that $h(X_k) =0$ for any $k \geq 0$ and $h(\widehat{M})=0$, which completes the proof.

\end{proof}

\end{document}